\documentclass{article}

\usepackage{microtype}
\usepackage{graphicx}
\usepackage{subfigure}
\usepackage{booktabs} 

\usepackage{hyperref}

\usepackage[accepted]{icml2025}

\usepackage{amsmath}
\usepackage{amssymb}
\usepackage{mathtools}
\usepackage{amsthm}

\usepackage[capitalize,noabbrev]{cleveref}

\usepackage{multirow}
\usepackage{stfloats}
\usepackage{bm}
\allowdisplaybreaks[4]
\usepackage{makecell}
\theoremstyle{plain}
\newtheorem{theorem}{Theorem}

\newtheorem{lemma}{Lemma}

\theoremstyle{definition}

\newtheorem{assumption}{Assumption}
\newtheorem{remark}{Remark}

\usepackage[textsize=tiny]{todonotes}

\icmltitlerunning{}

\begin{document}

\twocolumn[
\icmltitle{TeZO: Empowering the Low-Rankness on the Temporal Dimension in the Zeroth-Order Optimization for Fine-tuning LLMs}





\begin{icmlauthorlist}
\icmlauthor{Yan Sun}{1}
\icmlauthor{Tiansheng Huang}{}
\icmlauthor{Liang Ding}{1}
\icmlauthor{Li Shen}{2}
\icmlauthor{Dacheng Tao}{3}
\end{icmlauthorlist}

\icmlaffiliation{1}{Faculty of Engineering, The University of Sydney}
\icmlaffiliation{2}{School of Cyber Science and Technology, Shenzhen Campus of Sun Yat-sen University}
\icmlaffiliation{3}{The college of Computing \& Data Science, Nanyang Technological University}

\icmlcorrespondingauthor{Li Shen}{mathshenli@gmail.com}

\icmlkeywords{LLM finetuning, zeroth-order optimization, low-rankness}

\vskip 0.3in
]



\printAffiliationsAndNotice{} 

\begin{abstract}

Zeroth-order optimization~(ZO) has demonstrated remarkable promise in efficient fine-tuning tasks for Large Language Models~(LLMs). In particular, recent advances incorporate the low-rankness of gradients, introducing low-rank ZO estimators to further reduce GPU memory consumption. However, most existing works focus solely on the low-rankness of each individual gradient, overlooking a broader property shared by all gradients throughout the training, i.e., all gradients approximately reside within a similar subspace. In this paper, we consider two factors together and propose a novel low-rank ZO estimator, {\ttfamily TeZO}, which captures the low-rankness across both the model and temporal dimension. Specifically, we represent ZO perturbations along the temporal dimension as a 3D tensor and employ Canonical Polyadic Decomposition~(CPD) to extract each low-rank 2D matrix, significantly reducing the training cost. {\ttfamily TeZO} can also be easily extended to the {\ttfamily Adam} variant while consuming less memory than {\ttfamily MeZO-SGD}, and requiring about only $35\%$ memory of {\ttfamily MeZO-Adam}. Both comprehensive theoretical analysis and extensive experimental research have validated its efficiency, achieving SOTA-comparable results with lower overhead of time and memory.

\end{abstract}

\section{Introduction}
\label{tx:introduction}
As the model size progresses at an extraordinary rate \cite{zhang2022opt,touvron2023llama,achiam2023gpt}, memory and computational resources have become the primary bottleneck limiting development. In response to this challenge, ZO has opened up new possibilities for efficient training~\cite{shen2023efficient}. Adopting gradient-free updates with a small amount of high-quality data perfectly unlocks the knowledge of the entire domain, offering significant potential for several practical applications. Since \citet{spall1992multivariate} introduced ZO as a promising alternative to FO in the training process, it has been widely applied in gradient-computation-challenged scenarios \cite{wang2018stochastic,liu2020primer} and in black-box optimization \cite{chen2017zoo,tu2019autozoom}. Recent studies have also highlighted the great potential of ZO in fine-tuning LLMs. \citet{malladi2023fine} propose the {\ttfamily MeZO} method which adopts classical {\ttfamily ZO-SGD} \cite{ghadimi2013stochastic} for fine-tuning. Furthermore, it reduces memory costs by only preserving random seeds instead of variables. Compared to FO, it can achieve comparable performance while requiring approximately 10\% of memory in practice, greatly improving memory efficiency.

\begin{figure*}[t]
\centering
    \subfigure[Low-rankness on \textit{model} dimension.]{
	\includegraphics[width=0.33\textwidth]{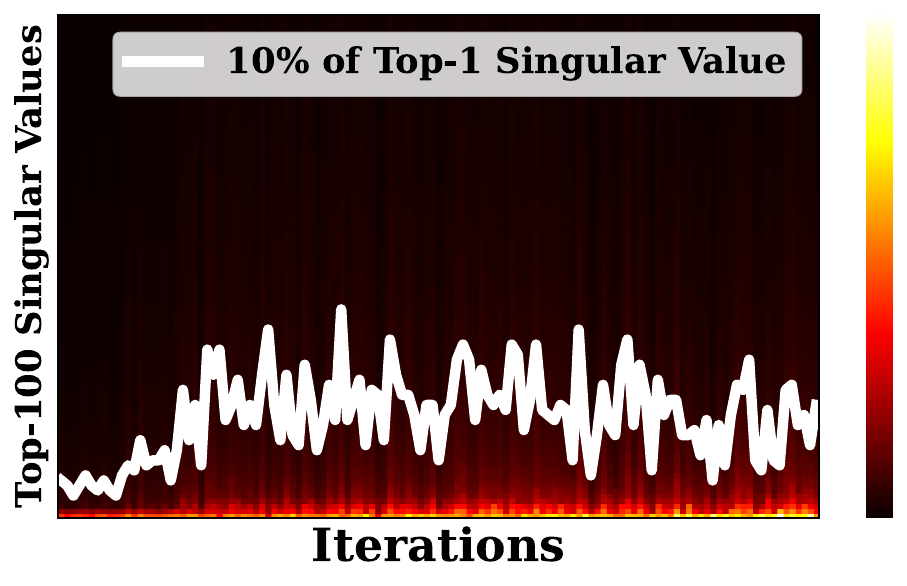}}\!\!\!
    \subfigure[Low-rankness on \textit{temporal} dimension.]{
	\includegraphics[width=0.33\textwidth]{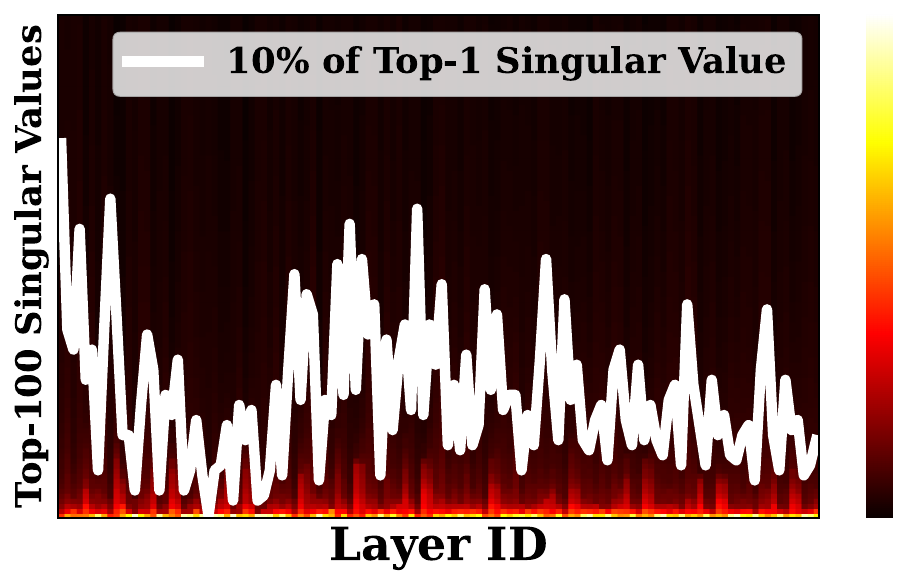}}\!\!\!
    \subfigure[Memory usage on training OPT-13B.]{
	\includegraphics[width=0.33\textwidth]{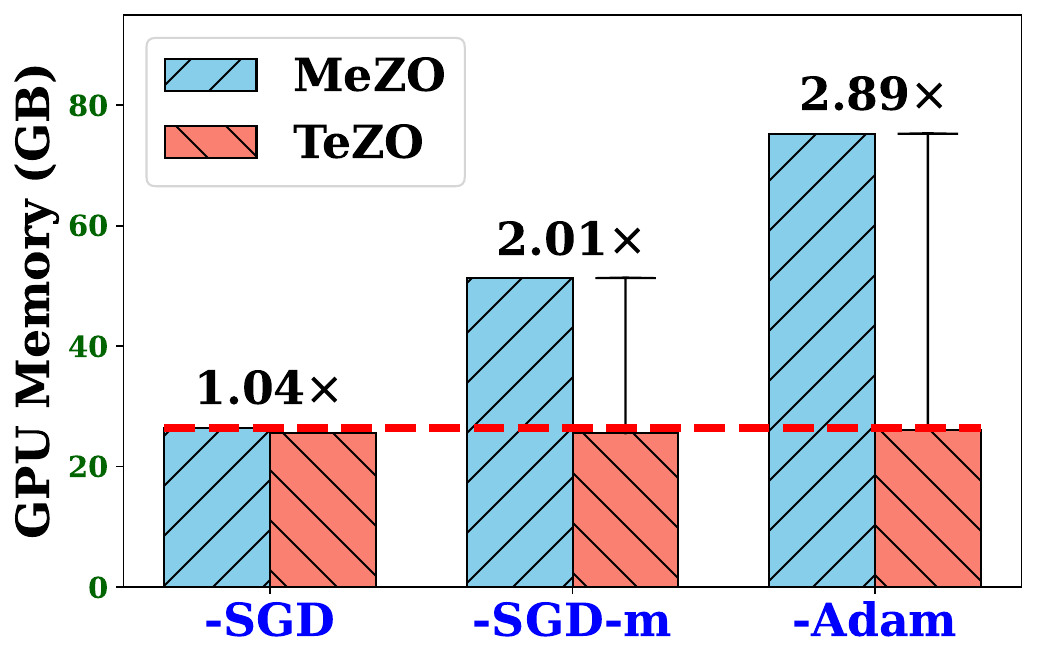}}
\vskip -0.1in
\caption{(a) and (b) are validation of the low-rankness of gradients. We fine-tune OPT-1.3B on SST-2 and calculate top-100 singular values of gradients of {\textnormal{\ttfamily layers.9.self\_attn.out\_proj.weight}}. We then concatenate these singular value vectors and display them as a heat-map in (a). Then we concatenate the normalized gradient of each layer over a total of $T$ iterations into a matrix with the size of $d_l\times T$, calculate the top-100 singular values corresponding to layers and display them as a heat-map in (b). In (c), we record the GPU memory usage of {\ttfamily MeZO}, our {\ttfamily TeZO}, and corresponding variants on training OPT-13B model. We also provide more interesting experiments on the low-rankness and studies of subspace of gradients on LLaMA-7B in Appendix~\ref{ap:low_rank}.}
\vskip -0.1in
\label{fg:intro}
\end{figure*}

Although ZO has made significant progress, it still faces two main challenges, i.e., i) lack of detailed characterization of gradients; ii) the costs of optimization states to generate random variables significantly increase as $d$ grows. This also highlights the bottleneck of ZO methods in LLM tasks. Recent advances learned the strong low-rank nature of gradients in LLMs~\cite{wang2023cuttlefish,jaiswal2024galore}, making low-rank representations in ZO methods as an ingenious solution to the aforementioned issues. With barely compromising performance, low-rank ZO methods effectively reduce the required memory for ZO estimations from $\mathcal{O}(d)$ to $\mathcal{O}(\sqrt{d}r)$ at most, where $r$ is the rankness constant \cite{chen2024enhancing,yu2024subzero}. This implementation further endows the ZO method with superior value in the tasks of fine-tuning LLMs. Nonetheless, this technique still exhibits substantial potential for further developments.

\textbf{Our Motivations.} Existing methods only consider each individual gradient to be low-rank. From the entire training perspective, the total cost achieves $\mathcal{O}(\sqrt{d}T)$ after $T$ iterations. Although techniques like lazy updates reduce the frequency of generating factor vectors, training costs still expand linearly with $\sqrt{d}$. It indicates that merely considering the low-rank form of each individual gradient is insufficient. This inspires our contemplation: \textit{could the low-rankness be incorporated along the temporal dimension as well?}

To investigate an efficient approach to addressing this question, in this paper, we comprehensively study the characteristics of the gradients in LLMs. As shown in \textit{Figure~\ref{fg:intro}.(a) and (b)}, in the tasks of fine-tuning LLMs, the gradients exhibit the following two properties simultaneously: i) the individual gradient at each iteration is approximately low-rank; ii) all gradients along $T$ iterations lie almost within a similar subspace. Obviously, combining properties i) and ii) can lead to higher efficiency. Inspired by this, we propose the {\ttfamily TeZO} estimator to empower the low-rankness on the temporal dimension. Specifically, we estimate the ZO perturbations as a 3D tensor with the size of $m\times n\times T$. By adopting the Canonical Polyadic Decomposition~(CPD) \cite{hitchcock1927expression}, the 3D tensor can be estimated by the sum of $r$ rank-1 tensors where $r$ approximates its rank. The joint low-rankness significantly reduces the cost of factor vectors during computation. At each iteration $t$, we only need to generate temporal factor vector to extract a 2D matrix, further lowering the costs from $\mathcal{O}(\sqrt{d}T)$ to $\mathcal{O}(\sqrt{d}+T)$. We also introduce an auxiliary technique to dynamically select the rank $r_l$ for each layer. In addition, we propose both memory-efficient momentum-based and adaptivity-based variants, i.e. {\ttfamily TeZO-m} and {\ttfamily TeZO-Adam}, which demonstrate the high scalability of the {\ttfamily TeZO} method. As shown in \textit{Figure~\ref{fg:intro}.(c)}, {\ttfamily TeZO-Adam} consumes less memory
than {\ttfamily MeZO-SGD}, and requiring about only $35\%$ memory of {\ttfamily MeZO-Adam}. Both comprehensive theoretical analysis and extensive experimental research are conducted to validate its efficiency. {\ttfamily TeZO-Adam} achieves performance superior to {\ttfamily MeZO-Adam} and other ZO optimizers, while requiring only the memory overhead of general {\ttfamily ZO-SGD}.
We summarize our contributions as follows:
\begin{itemize}
    \item By jointly considering low-rankness of gradients and their similarity on LLMs, we propose a novel low-rank ZO estimator, {\ttfamily TeZO}, which constructs the ZO perturbations via CPD to reduce the training overhead.
    \item We introduce an auxiliary technique to dynamically select the rank for each layer and extend {\ttfamily TeZO} to two memory-efficient variants, {\ttfamily TeZO-m} and {\ttfamily TeZO-Adam}. Both require lower memory consumption and time cost than {\ttfamily MeZO-SGD} and significantly lower than the costs of corresponding variants in {\ttfamily MeZO}.
    \item We prove that {\ttfamily TeZO} is an unbiased estimator of FO gradient, maintaining the comparable variance and convergence rate as existing ZO methods with less overhead. Extensive experiments are conducted to validate its fine-tuning efficiency on LLMs.
\end{itemize}
\section{Related Work}
\label{tx:related work}
\textbf{Zero-Order Optimization.} 
Since \citet{spall1992multivariate} proposed the ZO method, it has been extensively studied and practically incorporated in various domains \cite{chen2017zoo,tu2019autozoom,vemula2019contrasting,hajinezhad2019zone,gratton2021privacy}. By avoiding the massive computation and memory requirements of BP, it significantly reduces the training cost while maintaining high performance. As an alternative to FO, it has also been widely explored from several optimization perspectives, e.g. convergence for convex and non-convex \cite{wang2018stochastic,golovin2019gradientless,cheng2021convergence}, non-smooth \cite{liu2018zeroth,kazemi2024efficient,rando2024optimal}, variance reduction \cite{liu2018zeroth1,ji2019improved} and primal dual methods \cite{liu2018zeroth,yi2021linear,huang2024nonconvex}. It has also demonstrated strong potential for applications in certain practical scenarios, e.g. attack and defense \cite{zhao2020towards,kariyappa2021maze}, privacy protection \cite{gratton2021privacy,zhang2023dpzero,gupta2024inherent}, fairness \cite{chen2023deepzero,wang2024maft}, multi-agent \cite{tang2020distributed,maritan2023zo}, and efficient training \cite{nikolakakis2022black,fang2022communication,mukhoty2023direct}. These developments highlight the powerful potential of ZO methods in deep learning and artificial intelligence.

\textbf{Fine-tuning LLMs with ZO.} In this paper, we focus on the tasks of fine-tuning LLMs. Recent research on LLMs has demonstrated their immense value \cite{brown2020language,kojima2022large}. However, expensive time and memory costs in the training have become a significant barrier and hinder the research and application \cite{zhao2023survey,naveed2023comprehensive}. To unlock the tremendous potential of LLMs, researchers focus more on the training efficiency, leading to significant progress. The application of ZO optimizers has become a shining star from an optimization perspective. Since \citet{malladi2023fine} introduce the {\ttfamily MeZO} method, a series of ZO optimizer variants have been widely developed. \citet{jiang2024zo,yang2024adazeta,zhao2024second,zhao2024helene} focus on incorporating adaptivity and curvature information to accelerate ZO optimizers for LLMs.
\citet{liu2024sparse,guo2024zeroth,wang2024simultaneous} incorporate the sparsity to further reduce the calculations.
\citet{gautam2024variance} expand the variance reduction ZO estimator and evaluate its improvements in fine-tuning LLMs.
These methods improve ZO methods from the general optimization perspective, yield additional computational and memory overhead. Recently,
\citet{yu2024subzero,chen2024enhancing} further learn the low-rankness of each single gradient and propose different low-rank ZO estimators. These insightful works have advanced the application of ZO in fine-tuning LLMs.

\textbf{Our work.} Existing works focus on improving the calculation on single-step ZO gradient individually, ignoring the low-rankness in the gradient subspace. Our work extends this in the ZO estimator and achieves higher efficiency.

\section{Preliminaries}
\label{tx:preliminaries}

In this section, we introduce notations and review developments of ZO and its recent advances in fine-tuning LLMs.

\textbf{Notations.} We use lowercase letters to represent 1D vectors, e.g. $z$, uppercase letters to represent 2D matrices, e.g. $Z$, and bold uppercase letters to represent 3D tensors, e.g. $\bm{Z}$. Scalars are represented as lowercase Greek letters, e.g. $\alpha$. Other special computation symbols will be introduced in detail when they are first mentioned. \textit{Table~\ref{tb:notations}} shows some specific notation adopted in this paper. 

\begin{table}[b]
\centering
\vskip -0.2in
\caption{Some notations adopted in the context.}
\vskip 0.05in
\label{tb:notations}
\begin{tabular}{cc}
\toprule
\midrule
 Symbol & Notations \\ 
\midrule
 $w\ /\ W$ & learnable parameters \\
 $f(\cdot)$ & the general objective in the training \\
 $\nabla f(w)$ & FO gradient w.r.t. $w$ \\
 $\nabla^0 f(w)$ & ZO gradient w.r.t. $w$ \\
 $r$ & rank value for the ZO perturbation \\
 $\eta$ & global learning rate \ /\ step size \\
 $\rho$ & perturbation rate in the ZO estimation \\
\midrule
\bottomrule
\end{tabular}
\end{table}

\textbf{ZO Optimizer.} We consider the general and classical minimization problem on the task of fine-tuning LLMs:
\begin{equation}
    \min_w f(w)\triangleq \mathbb{E}_{\xi\sim\mathcal{D}}\left[f(w,\xi)\right],
\end{equation}
where $w\in\mathbb{R}^d$ is the learnable parameters and $\xi$ is the fine-tuning dataset sampled from the distribution $\mathcal{D}$. In this paper, we focus on a classical and widely adopted ZO method, \textit{Simultaneous Perturbation Stochastic Approximation} (SPSA) \cite{spall1992multivariate}. Specifically, SPSA estimates ZO gradient as:
\begin{equation}
\label{eq:zo_gradient}
    \nabla^0 f(w,\xi) = \frac{f(w + \rho z,\xi) - f(w - \rho z,\xi)}{2\rho} z,
\end{equation}
where $z\sim\mathcal{N}(0,I_d)$ is a random variable and $\rho$ is the perturbation rate. The estimation precision highly depends on the selection of the perturbation $\rho$. When $\rho$ is small enough, Eq.(\ref{eq:zo_gradient}) achieves an unbiased estimate of the FO gradient and efficiently drives the training process. Through two forward passes, it measures the projection component of the true gradient in the direction of the random variable $z$.

\textbf{Fine-tuning LLMs with ZO.} {\ttfamily MeZO} \cite{malladi2023fine} explores the tremendous potential of ZO methods in fine-tuning LLMs. Moreover, to reduce memory usage, it leverages PyTorch's permutation feature in random libs, replacing the storage of all random variables by recording the initial random seed for each iteration, namely the \textit{resampling technique}. This implementation enables the ZO method to achieve up to a $12\times$ memory saving in fine-tuning LLMs. The simple {\ttfamily ZO-SGD} method is sufficient to achieve performance comparable to FO methods in most tasks, enabling the training of ultra-large models like OPT-13B to be efficiently carried out on just a single A100 device. This has significantly advanced research in fine-tuning LLMs.

\textbf{Low-rank ZO.} Generally speaking, the parameter dimension of LLMs is extremely large, which constitutes a new bottleneck for the further development of {\ttfamily MeZO}: training costs of the ZO gradients increase linearly with the model dimension $d$. Furthermore, an important fact in fine-tuning LLMs is also ignored: the low-rankness of the gradients. Therefore, the applications of low-rank ZO techniques have emerged. \citet{chen2024enhancing} propose to apply matrix factorization as $Z=UV^\top$ ($Z\in\mathbb{R}^{m\times n}, U\in\mathbb{R}^{m\times r}, V\in\mathbb{R}^{n\times r}$). Additionally, \citet{yu2024subzero} adopt the form $Z=U\Sigma V^\top$ where $\Sigma\in\mathbb{R}^{r\times r}$, as shown in \textit{Figure \ref{fig:tezo}}. These techniques estimate the low-rank form of each individual perturbation per iteration, reducing the training cost of the ZO method in fine-tuning LLMs. Inspired by these insightful works, we examine another important aspect that is overlooked in the designs of previous works, i.e., low-rankness on the temporal dimension. Through the joint low-rank estimation, we propose the {\ttfamily TeZO} method which can further improve the efficiency of the ZO method in fine-tuning LLMs. Further discussions are provided in the next section.

\section{Methodology}
\label{tx:methodology}

\begin{figure*}[t]
\centering
\includegraphics[width=0.95\textwidth]{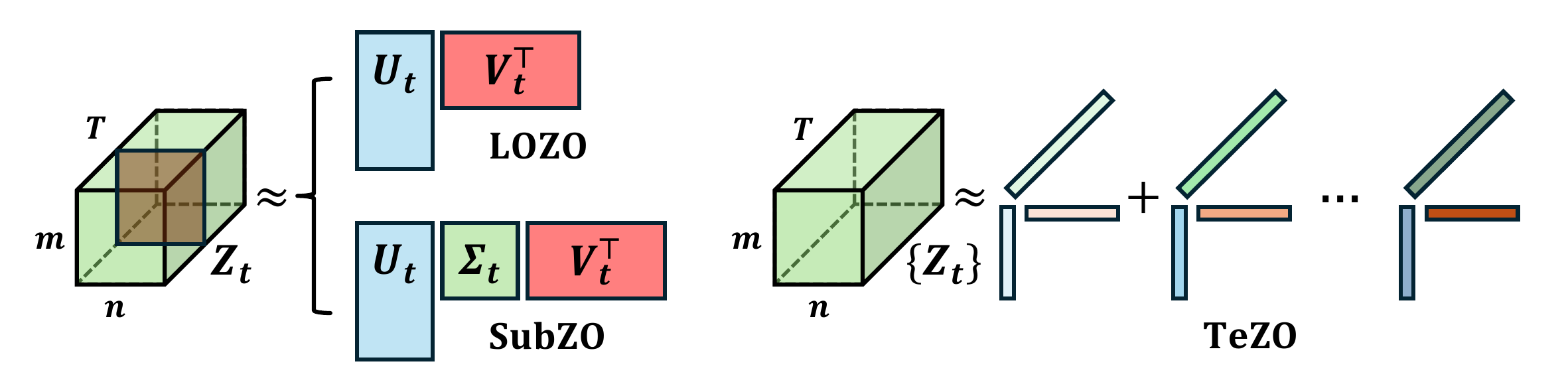}
\vskip -0.1in
\caption{The ZO diagrams for {\ttfamily LOZO}, {\ttfamily SubZO}, and our {\ttfamily TeZO} method. {\ttfamily LOZO} and {\ttfamily SubZO} focus on estimating a single perturbation $Z_t$ as the product of low-rank matrices. {\ttfamily TeZO} construct the entire perturbation set ${\bm Z} = \{Z_t\}$ via the CPD in the 3D tensor.}
\label{fig:tezo}
\end{figure*}

In this section, we introduce our proposed {\ttfamily TeZO} method. Then we introduce the adaptive selection of the rank of layer-wise gradients.
Finally, we present the momentum-based and adaptivity-based extensions of the {\ttfamily TeZO} method.

\subsection{Canonical Polyadic Decomposition and {\ttfamily TeZO}}

\textit{Canonical Polyadic Decomposition} \cite{hitchcock1927expression}, also known as \textit{Parallel Factor Analysis}, is the tensor decomposition technique widely used in data analysis, signal processing, and machine learning. It is a generalization of matrix factorization to higher-order tensors (multi-dimensional arrays). CPD aims to decompose a 3D tensor $\bm Z\in\mathbb{R}^{m\times n\times T}$ into a sum of rank-one tensors and each rank-one tensor is expressed as the outer product of vectors:
\begin{equation}
\label{eq:cpd}
\begin{split}
    {\bm Z}& \approx \sum_{s=1}^{r} \chi_s \circ u_s \circ v_s, \\ \text{where}& \ \ Z_t \approx \sum_{s=1}^{r} \tau_s \cdot \left(u_s \circ v_s\right),
\end{split}
\end{equation}
where $\tau_s=(\chi_s)_t$ at time $t$ and $\circ$ denotes the outer product. $\chi_s\in\mathbb{R}^T, u_s\in\mathbb{R}^m, v_s\in\mathbb{R}^n$ are three factor vectors.

\begin{table}[b]
\centering
\vskip -0.1in
\caption{The number of the sampled elements for training a 2D weight ($m\times n=d$) after $T$ iterations with the corresponding ZO.}
\vskip 0.05in
\label{tb:generate number}
\begin{tabular}{ccc}
\toprule
\midrule
 Method & Total elements of generating & Optimal\\ 
\midrule
 {\ttfamily MeZO} & $mnT$ & $\mathcal{O}(d\cdot T)$ \\
 {\ttfamily SubZO} & $(m+n+r)rT$ & $\mathcal{O}(\sqrt{d}\cdot T)$\\
 {\ttfamily LOZO} & $(m+n)rT$ & $\mathcal{O}(\sqrt{d}\cdot T)$\\
\midrule
 {\ttfamily TeZO} & $(m + n + T)r$ & $\mathcal{O}(\sqrt{d} + T)$\\
 \midrule
\bottomrule
\end{tabular}
\end{table}

Based on understanding the low-rank nature along the temporal dimension, we propose a novel low-rank estimation approach to represent the gradient perturbation variables at each iteration, as outlined in \textit{Eq.(\ref{eq:cpd})}. In LLMs, the proportion of 2D model parameters is much larger than that of 1D parameters, so we primarily consider the 2D cases. Specifically, in addition to the conventional factor vectors $u$ and $v$ for the model dimensions, we introduce the factor vectors $\chi$ for the temporal dimension. These three dimensions are independent of each other. Both $u$ and $v$ can be initialized at the beginning of training. Therefore, in the $t$-th iteration, we only need information related to the variable $\tau_s$ without any additional redundant variables, which can still significantly reduce training costs. Compared to existing studies, we summarize the results in Table~\ref{tb:generate number}.

\subsection{Layer-wise Selection of the Rank $r$}
The selection of rank $r$ remains an open challenge. Since ZO methods are typically employed in scenarios where FO gradients are unavailable, it is difficult to directly determine the precise rank of gradients. Recent studies have emphasized the feasibility of low-rank structures, and the rank $r$ is empirically treated as a constant hyperparameter. In fact, $r$ essentially represents a trade-off between performance and efficiency, and selecting an appropriate value for $r$ can significantly enhance the balance between these two factors. Although the constant selection can yield reliable performance, our goal is to identify a more refined solution.

To comprehensively study the rank $r$, we learn its potential connection to the model parameters in this paper. Furthermore, based on this understanding, we dynamically select the rank on different layers in the {\ttfamily TeZO} method. We consider the general cascade neural network as following:
\begin{equation}
    X_l = \sigma_l(A_l), \quad A_l=W_l X_{l-1} + b_l,
\end{equation}
where $\sigma_l(\cdot)$ is activation function, $l\leq L$ is the index of each layer, $W_l$ and $B_l$ are the weight and bias of the $l$-th layer. Let the final output simply be $X_L$, therefore, we can calculate the gradient of the parameters $W_l$ as:
\begin{equation}
    \frac{\partial f}{\partial W_l} = \left(\prod_{p=l+1}^{L}W_p\right)^\top \partial\Phi(\sigma_L,\sigma_{L-1},\cdots,\sigma_l),
\end{equation}
where $\partial\Phi(\cdot)$ are the joint gradients for all activations from the total mini-batch data samples, whose rank is closely related to the similarity of the input data. In this paper, we focus on the impact from model parameters. According to the rank propagation, the rankness of each gradient satisfies:
\begin{equation}
\label{eq:low rank gradient}
\begin{split}
    &\quad \text{Rank}\left(\frac{\partial f}{\partial W_l}\right) \leq \text{Rank}\left(\prod_{p=l+1}^{L}W_p\right) \\
    &\leq \min\left(\text{Rank}\left(W_{l+1}\right),\cdots,\text{Rank}\left(W_L\right)\right).
\end{split}
\end{equation}
Typically, during training, due to the use of weight decay regularization, the model parameters tend to maintain a high degree of low-rankness. Therefore, the gradients also inherit this property, meaning that the low-rankness of the gradients originates from the low-rankness of the model parameters. We adaptively determine the rank of different layers based on the insight from Eq.(\ref{eq:low rank gradient}). In LLMs, there is a natural cascade block structure, where each block contains components such as the attention module and the feed-forward network (FFN) module. We adopt the truncated Eq.(\ref{eq:low rank gradient}) to estimate the rankness of each layer within a block. Specifically, we split $L$ into $B$ blocks as $[\{l_0\}, \{l_1\}, \cdots, \{l_B\}]$ where $\sum_b l_b=L$. The rankness of the gradient of the $l$-th layer will be estimated as follows:
\begin{equation}
\label{eq:truncated low rank gradient}
    r_l = \min\left(\{\text{Rank}\left(W_{\{l_b\}}\right)\}, r_{max}\right),
\end{equation}
where $l\in\{l_b\}$ and $r_{max}$ is a constant. Eq.(\ref{eq:truncated low rank gradient}) is designed to preserve the transitivity of such estimations on the rankness, ensuring that the estimated $r$ does not become excessively low due to very large $L$.
$\text{Rank}(W)$ is defined as the index of the top-$r$ singular values of the matrix $W$. The selection criterion can be based on the proportion of total energy of singular values or the percentage of the largest singular value to represent its low-rank property. In our experiments, we uniformly set a specific threshold to determine $r$ that those singular values are larger than that threshold. Through Eq.(\ref{eq:truncated low rank gradient}), {\ttfamily TeZO} enables the estimation of the rank of FO gradients without explicitly calculating the FO gradients. 

\subsection{Applications in General Optimizers}
\begin{algorithm}[t]
\renewcommand{\algorithmicrequire}{\textbf{Input:}}
\renewcommand{\algorithmicensure}{\textbf{Output:}}
\caption{{\ttfamily TeZO/TeZO-m/TeZO-Adam} Methods}
\label{tx:algorithm}
\begin{algorithmic}[1]
    \REQUIRE model $\left\{W_l\right\}$, perturbation rate $\rho$, learning rate $\eta_l$, iterations $T$, momentum coefficient $\beta_1=0.9$, second-order momentum coefficient $\beta_2=0.99$, smoothing term $\epsilon=1e-5$.
    \ENSURE model $\left\{W_l\right\}$.
    \STATE Initialize the rank list $[r_1, \cdots, r_L]$ via Eq.(\ref{eq:truncated low rank gradient})
    \STATE Initialize the factor vectors $\{u_s\}$ and $\{v_s\}$ by layers
        \FOR{$t = 0, 1, 2, \cdots, T-1$}
            \STATE select the minibatch $\xi_t$ and random seed $\zeta_t$
            \STATE $W = \text{Perturbation}(W,\rho,\zeta_t)$, \ \ $f_+ = f(W,\xi)$
            \STATE $W = \text{Perturbation}(W,-2\rho,\zeta_t)$, \ \ $f_- = f(W,\xi)$
            \STATE $W = \text{Perturbation}(W,\rho,\zeta_t)$, \ \ $\kappa_t = (f_+ - f_-) / 2\rho$
            \STATE reset the random seed as $\zeta_t$
            \FOR{$W_l \in W$}
                \STATE sample $\tau\sim\mathcal{N}(0,I_{r_l})$ \ \ 
                \STATE {\ttfamily (TeZO)} $\quad \quad \quad G_t = \sum_{s=1}^{r_l}\kappa_t\tau_s\cdot\left(u_s \circ v_s\right)$
                \STATE {\ttfamily (TeZO-m)} $\quad \ \ \ \tau_M = \beta_1\tau_M + (1-\beta_1) \kappa_t \tau$
                \STATE {\color{white}{\ttfamily (TeZO-m)}} $\quad \ \ \ G_t = \sum_{s=1}^{r_l}(\tau_M)_s\cdot\left(u_s \circ v_s\right)$
                \STATE {\ttfamily (TeZO-Adam)} $\tau_M = \beta_1\tau_M + (1-\beta_1) \kappa_t \tau$
                \STATE {\color{white}{\ttfamily (TeZO-Adam)}} $\tau_V = \beta_2\tau_V + (1-\beta_2)\kappa_t^2\tau^2$
                \STATE {\color{white}{\ttfamily (TeZO-Adam)}} $M_t = \sum_{s=1}^{r_l}(\tau_M)_s\cdot\left(u_s \circ v_s\right)$
                \STATE {\color{white}{\ttfamily (TeZO-Adam)}} $V_t = \sum_{s=1}^{r_l}(\tau_V)_s\cdot\left(u_s^2 \circ v_s^2\right)$
                \STATE {\color{white}{\ttfamily (TeZO-Adam)}} $G_t = M_t/\sqrt{V_t + \epsilon}$
                \STATE $W_l=W_l-\eta_l G_t$
            \ENDFOR
        \ENDFOR
    \STATE \textbf{Function} $\text{Perturbation}(W,\rho,\zeta)$:
    \STATE reset the random seed as $\zeta$
    \FOR{$W_l \in W$}
    \STATE sample $\tau\sim\mathcal{N}(0,I_{r_l})$, \ \ $Z_t = \sum_{s=1}^{r_l}\tau_s\cdot\left(u_s \circ v_s\right)$
    \STATE $W_l = W_l + \rho Z_t$
    \ENDFOR
\end{algorithmic}
\end{algorithm}

In this part, we mainly introduce the application of {\ttfamily TeZO} in several classical optimizers, as shown in Algorithm~\ref{tx:algorithm}.

\textbf{{\ttfamily TeZO}}. {\ttfamily ZO-SGD} always serves as a foundational approach in previous works. Similarly, we adopt the resampling technique proposed by {\ttfamily MeZO} to reduce memory usage. Before each iteration, the random seed is reset to ensure sampling the same variables. Through three perturbations, we can calculate the positive and negative terms, i.e., $f_+=f(w+\rho z,\xi)$ and $f_-=f(w-\rho z,\xi)$ in Eq.(\ref{eq:zo_gradient}), and update the projected coefficient $\kappa = (f_+ - f_-)/2\rho$. It naturally requires only estimating the current $Z_t$ and then performing the update as a general ZO step.

\textbf{{\ttfamily TeZO-m}}. {\ttfamily SGD-m} method has also received widespread attention. The use of momentum allows this method to maintain greater stability during practical training. For the {\ttfamily MeZO} method, the momentum term requires an additional doubling of parameter storage, which undoubtedly increases the algorithm's cost. However, in the proposed {\ttfamily TeZO} method, this implementation becomes highly memory-efficient. Specifically, since the factor vectors of the dimensions are not affected by time, the momentum accumulation of $Z_t$ can be equivalently achieved by first applying momentum accumulation to $\kappa_t\tau_t$ and then computing the momentum term. {\ttfamily TeZO} achieves the global momentum updates with only additional storage of $\tau_M$ and is not affected by the model dimension $d$ in the training process.

\textbf{{\ttfamily TeZO-Adam}}. {\ttfamily Adam} is highly favored by researchers and demonstrates greater potential in training LLMs. Second-order momentum effectively scales the updates of coordinates with larger long-term changes in the gradient, allowing the model to adaptively adjust the learning rate for each coordinate, thus enabling efficient training. Clearly, the drawback is that it introduces more computational and storage demands. To strictly implement {\ttfamily Adam} for {\ttfamily TeZO}, we also face the storage issue of second-order momentum. To reduce the overhead, we propose a \textit{lightweight} variant where the second-order momentum is first computed separately along each factor vector and then merged, thus avoiding significant additional storage. Specifically, we review the squared gradient in second-order momentum for a 2D parameters $W_l$ on the $l$-th layer in our {\ttfamily TeZO} method :
\begin{equation}
\label{eq:second order}
\begin{split}
    &\quad \left[\nabla^0 f(W_l)\right]^2 =
    \kappa_t^2Z_t^2 = \kappa_t^2\left(\sum_{s=1}^{r_l}\tau_s\cdot(u_s\circ v_s)\right)^2 \\
    &= \underbrace{\sum_{s=1}^{r_l}\kappa_t^2\tau_s^2\cdot(u_s^2\circ v_s^2)}_{\text{Separable Term}} + \underbrace{\kappa_t^2\sum_{p\neq q}^{r_l} \tau_p\tau_q\cdot(u_p u_q\circ v_p v_q)}_{\mathbb{E}_{\tau,u,v}\left[\tau_p\tau_q\cdot(u_p u_q\circ v_p v_q) \right]={\bm 0}}.
\end{split}
\end{equation}
In the above equation, the second term, i.e., the cross term, has an overall expectation of zero on each coordinate. In practice, we test specific cases with different sizes of $u_s,v_s$ and selections of $r_l$ in our experiments. Compared with the first 
separable term, the impact of the second term becomes negligible. Therefore, our lightweight {\ttfamily TeZO-Adam} only accumulates the second-order momentum via the first term, i.e. the separable term. This allows us to calculate it via accumulation of the temporal factor vector, similar to how we compute first-order momentum. We first update the $\kappa_t^2\tau^2$ term and then expand it into the second-order momentum with $u_s^2$ and $v_s^2$. By this way, we only need to store the $\tau_V$ vector additionally during training, without being affected by the model dimension $d$. Due to page limitation, we provide more studies on the precision and efficiency about {\ttfamily TeZO-Adam} in Appendix~\ref{ap:tezo-adam}.

\textbf{Advantages of {\ttfamily TeZO}.} Existing works mainly consider the low-rank efficiency of the {\ttfamily ZO-SGD} method, while the extensions to other optimizers remain inefficient. {\ttfamily TeZO} not only takes into account the properties of a joint low-rankness, but also serves as an extension-friendly ZO design. It maintains both high memory and computational efficiency across several classical optimizers. This also brings significant benefits for fine-tuning LLMs in practice. 
\section{Theoretical Analysis}
\label{tx:theoretical analysis}


In this section, we mainly introduce the theoretical analysis of {\ttfamily TeZO}, including fundamental properties and convergence guarantees in the application of various optimizers. Due to space limitation, all proofs are detailed in the Appendix~\ref{ap:proofs}.
\begin{theorem}[Expectation and Variance]
\label{thm:mean and variance}
    Without loss of generality, we consider the 2D parameters $W\in\mathbb{R}^{m\times n}$. Its FO gradient is denoted as $\nabla_{W} f$ and ZO gradient is denoted as $\nabla_{W}^0 f$. When using the {\ttfamily TeZO} method to estimate the ZO gradient with rank $r$ and a sufficiently small perturbation rate $\rho$ as shown in Algorithm~\ref{tx:algorithm}, the following holds:
    \begin{equation}
    \label{eq:zo mean and variance}
    \begin{split}
        &\mathbb{E}_{\tau,u,v}\left[\frac{1}{r}\lim_{\rho\rightarrow 0}\nabla_{W}^0 f\right] = \nabla_{W} f, \\
        \mathbb{E}_{\tau,u,v}\Vert&\frac{1}{r}\lim_{\rho\rightarrow 0}\nabla_{W}^0 f - \nabla_{W} f\Vert^2 = \delta \Vert \nabla_{W} f \Vert^2,
    \end{split}
    \end{equation}
    where $\delta = 1 + mn + \frac{2mn}{r} + \frac{6(m+n)}{r} + \frac{10}{r}$.
\end{theorem}

\begin{remark}
    {\ttfamily TeZO} is an unbiased zero-order estimator and its variance is linearly correlated with the norm of the FO gradient. Moreover, we provide detailed relationships between the variance coefficient $\delta_l$ and the matrix sizes $m_l,n_l$ as well as rank $r_l$. Previous work~\cite{yu2024subzero} focuses on the impact of low-rankness on variance from the perspective of the subspace for the quadratic objective. We provide the formal expression under the low-rank representation for a general smooth objective. The variance for low-rank representation is slightly larger than that of the {\ttfamily MeZO} method, i.e. $mn$, remaining within the same order. This indicates that {\ttfamily TeZO} has comparable ability to {\ttfamily MeZO} in practice while requiring significantly less training costs.
\end{remark}

Then we consider the convergence. In this paper, we consider the general smooth and non-convex function under:
\begin{assumption}
\label{as:smooth}
    $f(\cdot)$ is a smooth and nonconvex objective, i.e., for $\forall x,y\in\mathbb{R}^d$, $\Vert \nabla f(x,\xi) - \nabla f(y,\xi) \Vert \leq \lambda\Vert x - y \Vert$.
\end{assumption}
\begin{assumption}
\label{as:stochastic}
    The stochastic gradient is an unbiased estimator with bounded variance, i.e., for each data sample $\xi$, $\mathbb{E}_{\xi}\left[\nabla f(x,\xi)\right] = \nabla f(x)$, $\mathbb{E}_{\xi}\Vert\nabla f(x,\xi) - \nabla f(x)\Vert^2\leq \sigma^2$.
\end{assumption}
These are two commonly adopted assumptions in ZO optimization. Prior works~\cite{chen2024enhancing,yu2024subzero} consistently impose the requirement that some or all factor vectors exhibit column orthogonality. In contrast, our proof does not rely on the need for such additional constraints.
\begin{theorem}[Convergence]
\label{thm:convergence}
    Without loss of generality, we consider the 2D parameters $W\in\mathbb{R}^{m\times n}$. Under Assumption~\ref{as:smooth} and \ref{as:stochastic}, let $\eta=\mathcal{O}\left(\sqrt{\frac{D_0}{\lambda T\left(\rho^2\lambda^2\delta_\rho + \delta\sigma^2\right)}}\right)\leq \frac{1}{\lambda(\delta+1)}$ where $D_0 = f(W_0)-f(W_\star)$ is the initialized bias, the sequence $\left\{W_t\right\}_{t=0}^{T-1}$ generated by {\ttfamily TeZO} converges as:
    \begin{equation}
        \frac{1}{T}\sum_{t=0}^{T-1}\mathbb{E}\Vert \nabla f(W_t) \Vert^2 = \mathcal{O}\left(\sqrt{\frac{\lambda D_0\left(\rho^2\lambda^2\delta_\rho + \delta\sigma^2\right)}{T}}\right),
    \end{equation}
    where $\delta_\rho=\frac{15r^2(m+3)^3(n+3)^3 + 36r^3m^3n^3 + r^4m^3n^3}{4}$ and $\delta$ is defined in Theorem~\ref{thm:mean and variance}.
\end{theorem}
\begin{remark}
    This convergence result maintains the same rate of recent ZO advances. By substituting the total parameters for $d$, we have $\delta=\mathcal{O}(d)$ and $\delta_\rho=\mathcal{O}(d^3)$. Let the perturbation rate $\rho = \mathcal{O}(\frac{\sigma}{\lambda}d^{-1})$, we have the final rate as $\mathcal{O}(\sqrt{\frac{\lambda D_0 d \sigma^2}{T}})$ which recovers the general rate of the recent ZO methods. 
    This also demonstrates the advantages of the {\ttfamily TeZO} method, as it reduces the complexity of random sample generation from $\mathcal{O}(d\cdot T)$ to $\mathcal{O}(\sqrt{d} + T)$ and effectively decreases memory usage, while theoretically maintaining the similar convergence rate.
\end{remark}

\section{Experiments}
\label{tx:experiments}

\begin{table*}[t]
\begin{center}
\renewcommand{\arraystretch}{1}
\caption{Experiments of fine-tuning for 80k iterations on RoBERTa-large and then perform evaluations. {\ttfamily FT}: FO fine-tuning (5 epochs with {\ttfamily Adam}). {\ttfamily ZERO-SHOT}: test only. \textbf{AVG.} measures the average gap across all datasets compared to {\ttfamily FT}~($+0$).}
\vskip 0.05in
\begin{sc}
\small
\setlength{\tabcolsep}{2.1mm}{\begin{tabular}{@{}c|ccccc|c|ccccc|c@{}}
\toprule
\midrule
\multicolumn{1}{c}{} & \multicolumn{6}{c}{$k=16$} & \multicolumn{6}{c}{$k=512$} \\
\cmidrule(lr){2-7} \cmidrule(lr){8-13}
\multicolumn{1}{c}{} & \multicolumn{1}{c}{{\bf SST-5}} & \multicolumn{1}{c}{{\bf SNLI}} & \multicolumn{1}{c}{{\bf MNLI}} & \multicolumn{1}{c}{{\bf QNLI}} & \multicolumn{1}{c}{{\bf TREC}} & \multicolumn{1}{c}{\bf{AVG.}} & \multicolumn{1}{c}{\bf{SST-5}} &\multicolumn{1}{c}{{\bf SNLI}} & \multicolumn{1}{c}{{\bf MNLI}} & \multicolumn{1}{c}{{\bf QNLI}} & \multicolumn{1}{c}{{\bf TREC}} & \multicolumn{1}{c}{{\bf AVG.}} \\
\cmidrule(lr){1-13} 
{\ttfamily FT} & 45.0 & 71.9 & 65.3 & 70.6 & 87.4 & $+0$ 
 & 57.5 & 88.3 & 84.2 & 87.5 & 97.2 & $+0$ 
 \\
{\ttfamily ZERO-SHOT} & 22.0 & 33.7 & 34.0 & 52.6 & 20.4 & $-35.5$ & 22.0 & 33.7 & 34.0 & 52.6 & 20.4 & $-50.4$  \\
\midrule
{\ttfamily MeZO} & 44.7 & 67.6 & 60.9 & 64.8 & 58.6 & $-8.7$ & 56.4 & 83.2 & 79.5 & 83.3 & 95.6 & $-3.3$ \\
{\ttfamily SubZO} & 44.8 & 65.7 & 62.8 & 64.7 & 56.6 & $-9.1$ & 55.7 & 83.1 & 80.1 & 83.7 & 95.4 & $-3.3$ \\
{\ttfamily LOZO} & 42.0 & 67.1 & 60.2 & 64.7 & 61.2 & $-9.0$ & 56.0 & 84.0 & 81.6 & 82.4 & 95.4 & $-3.1$ \\
{\ttfamily TeZO} & 42.8 & 67.6 & 61.8 & 64.1 & 57.4 & $-9.3$ & 54.7 & 84.0 & 79.3 & 82.7 & 95.8 & $-3.6$ \\
\midrule
{\ttfamily MeZO-m} & 44.5 & 67.6 & 62.1 & 65.9 & 61.1 & $-7.8$ & 56.6 & 83.4 & 79.9 & 83.4 & 95.4 & $-3.2$ \\
{\ttfamily LOZO-m} & 44.1 & 67.8 & 61.3 & 64.2 & 62.2 & $-8.1$ & 55.6 & 84.1 & 80.1 & 82.4 & 94.6 & $-3.6$ \\
{\ttfamily TeZO-m} & 43.8 & 67.0 & 61.1 & 65.0 & 60.4 & $-8.5$ & 54.5 & 84.4 & 79.9  & 82.9 & 95.5 & $-3.4$ \\
\midrule
\bottomrule
\end{tabular}}
\label{acc1}
\end{sc}
\end{center}
\vskip -0.15in
\end{table*}

\begin{table*}[t]
\begin{center}
\renewcommand{\arraystretch}{1}
\caption{Experiments of fine-tuning for 15k iterations on OPT-13B. Other setups are consistent with the \textit{Table \ref{acc1}}.}
\vskip 0.05in
\begin{sc}
\small
\setlength{\tabcolsep}{1.9mm}{\begin{tabular}{@{}c|ccccccccccc|c@{}}
\toprule
\midrule
\multicolumn{1}{c}{} & \multicolumn{1}{c}{\bf{SST-2}} & \multicolumn{1}{c}{{\bf RTE}} & \multicolumn{1}{c}{{\bf CB}} & \multicolumn{1}{c}{{\bf BoolQ}} & \multicolumn{1}{c}{{\bf WSC}} & \multicolumn{1}{c}{{\bf WIC}} & \multicolumn{1}{c}{{\bf MultiRC}} & \multicolumn{1}{c}{{\bf COPA}} & \multicolumn{1}{c}{{\bf ReCoRD}} & \multicolumn{1}{c}{{\bf SQuAD}} & \multicolumn{1}{c}{{\bf DROP}} & \multicolumn{1}{c}{{\bf AVG.}} \\
\cmidrule(lr){1-13}
{\ttfamily FT} & 91.5 & 70.7 & 84.0 & 76.4 & 63.5 & 70.0 & 71.1 & 79.0 & 74.1 & 84.7 & 31.5 & $+0$ \\ 
{\ttfamily ZERO-SHOT} & 58.5 & 59.4 & 46.4 & 59.1 & 38.3 & 55.2 & 46.7 & 80.0 & 81.0 & 46.6 & 14.4 & $-19.2$ \\
\midrule
{\ttfamily MeZO} & 90.1 & 60.3 & 67.9 & 66.1 & 62.5 & 54.7 & 57.7 & 87.0 & 81.1 & 79.6 & 30.4 & $-5.4$ \\
{\ttfamily SubZO} & 91.3 & 61.9 & 67.9 & 66.1 & 63.5 & 55.9 & 57.3 & 86.0 & 81.9 & 80.7 & 30.5 & $-4.8$ \\
{\ttfamily LOZO} & 90.3 & 62.5 & 67.9 & 65.6 & 63.5 & 55.3 & 56.9 & 86.0 & 81.8 & 80.5 & 30.1 & $-5.1$ \\
{\ttfamily TeZO} & 90.2 & 61.1 & 69.6 & 65.1 & 63.5 & 54.3 & 56.8 & 87.0 & 81.2 & 80.7 & 29.6 & $-5.2$ \\
\midrule
{\ttfamily MeZO-m} & 90.6 & 60.7 & 67.9 & 65.5 & 62.5 & 54.6 & 57.9 & 88.0 & 81.5 & 79.5 & 30.4 & $-5.2$ \\
{\ttfamily LOZO-m} & 90.7 & 62.1 & 67.1 & 65.7 & 62.5 & 55.7 & 57.7 & 88.0 & 81.7 & 80.7 & 29.8 & $-4.9$\\
{\ttfamily TeZO-m} & 91.1 & 61.4 & 69.6 & 64.6 & 63.5 & 55.6 & 56.7 & 88.0 & 81.3 & 80.9 & 30.0 & $-4.8$ \\
\midrule
{\ttfamily MeZO-Adam} & 92.4 & 70.5 & 67.9 & 70.0 & 62.5 & 58.7 & 58.9 & 88.0 & 81.1 & 80.8 & 30.5 & $-3.2$ \\
{\ttfamily ZO-AdaMU} & 92.0 & 72.9 & 67.9 & 71.0 & 61.5 & 59.7 & 58.4 & 86.0 & 81.5 & 82.4 & 31.1 & $-2.9$ \\
{\ttfamily TeZO-Adam} & 93.3 & 71.8 & 69.6 & 71.8 & 59.9 & 60.5 & 60.3 & 86.0 & 81.5 & 84.0 & 29.8 & $-2.5$ \\
\midrule
\bottomrule
\end{tabular}}
\label{acc2}
\end{sc}
\end{center}
\vskip -0.15in
\end{table*}

In this section, we mainly show the empirical studies. We follow the recent studies of fine-tuning LLMs tasks with ZO methods~\cite{malladi2023fine,yu2024subzero,chen2024enhancing,jiang2024zo} and adopt the similar setups to validate the efficiency. The main text primarily introduces baselines, performance evaluations, and training costs. Due to page limitations, other contents, including experimental details, hyperparameter selections and some interesting validation results, have been placed in Appendix~\ref{ap:experiments}.

\textbf{Baselines and setups.} We select recent advances of ZO and low-rank ZO methods on fine-tuning LLM tasks as baselines, including {\ttfamily MeZO} \cite{malladi2023fine}, {\ttfamily LOZO} \cite{chen2024enhancing}, {\ttfamily SubZO} \cite{yu2024subzero}, and their variants of momentum-based and {\ttfamily Adam}-based extensions in their works. We also compare {\ttfamily ZO-AdaMU} \cite{jiang2024zo} which focuses on adaptivity. Similar to these works, we conducted tests on different models, including RoBERTa-large \cite{liu1907roberta}, OPT \cite{zhang2022opt}, and LLaMA \cite{touvron2023llama}. We select a total of 16 datasets for testing and compute the final average performance to fairly compare the overall efficiency of each method.

\begin{figure*}[t]
\centering
    \subfigure[GPU Memory usage on fine-tuning OPT-13B.]{
	\includegraphics[width=0.49\textwidth]{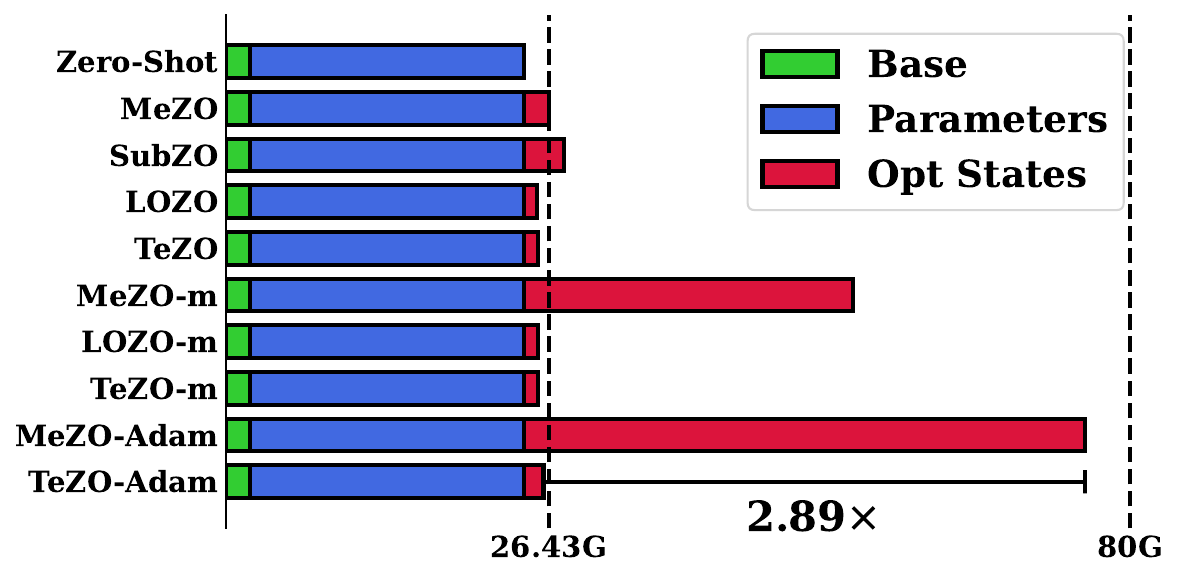}}
    \subfigure[Wall-clock time\ /\ iteration on fine-tuning OPT-13B.]{
	\includegraphics[width=0.49\textwidth]{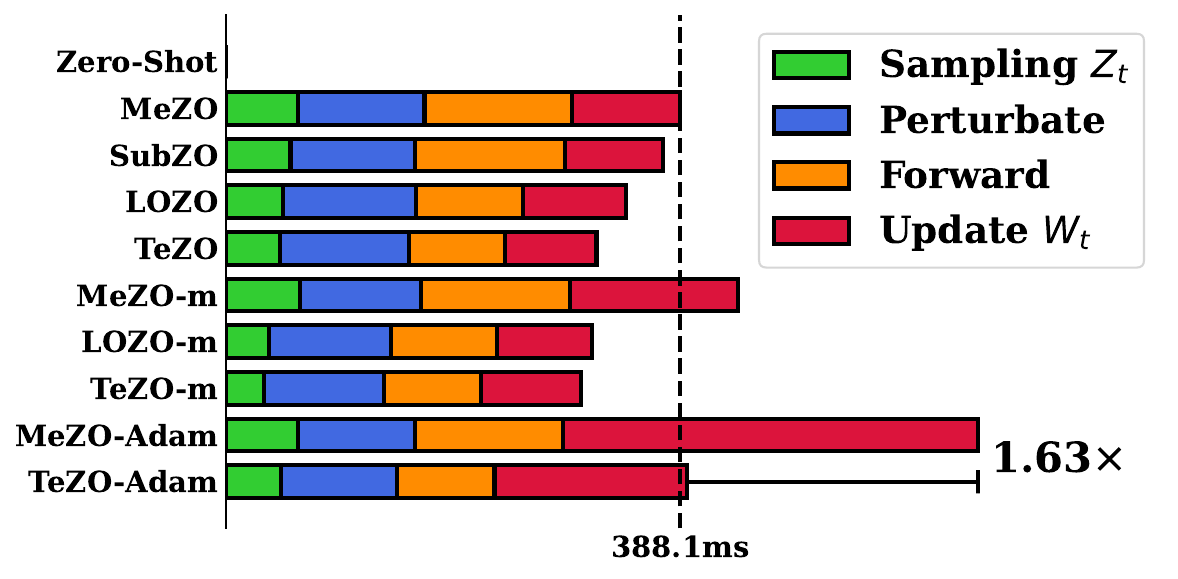}}
\vskip -0.1in
\caption{GPU memory usage (a) and wall-clock time (b) for fine-tuning LLMs with RTE dataset on H100. {\ttfamily SubZO} does not provide other memory-efficient extensions. {\ttfamily LOZO} does not provide the memory-efficient {\ttfamily Adam} extension. More details are stated in Appendix~\ref{ap:memory and time}.}
\label{fg:memory and time}
\end{figure*}

\textbf{Medium-sized Models.} We conduct the experiments on the RoBERTa-large model for the general sentiment classification, natural language inference and text retrieval tasks, as shown in \textit{Table \ref{acc1}}. To eliminate experimental randomness, the reported results are the averages of 5 runs with different random seeds. It clearly demonstrates the efficiency of the ZO method on medium-sized models. In fact, for medium-sized models, {\ttfamily MeZO} remains the most accurate ZO method. However, the gap between other low-rank methods and {\ttfamily MeZO} is not significant, not exceeding 0.6\% at the worst scenario on average. In the environment with $k=512$, almost all ZO methods show no significant differences.

\begin{table}[t]
\begin{center}
\renewcommand{\arraystretch}{1}
\vskip -0.1in
\caption{Experiments of fine-tuning 15k iterations on LLaMA-7B. Other setups are consistent with the \textit{Table \ref{acc1}}.}
\vskip 0.05in
\begin{sc}
\small
\setlength{\tabcolsep}{2.1mm}{\begin{tabular}{@{}c|cccc|c@{}}
\toprule
\midrule
\multicolumn{1}{c}{} & \multicolumn{1}{c}{{\bf SST-2}} & \multicolumn{1}{c}{{\bf RTE}} & \multicolumn{1}{c}{{\bf WSC}} & \multicolumn{1}{c}{{\bf WIC}} & \multicolumn{1}{c}{\bf{AVG.}} \\
\cmidrule(lr){1-6} 
{\ttfamily FT} & 95.6 & 86.3 & 64.4 & 70.4 & $+0$ \\ 
{\ttfamily ZERO-SHOT} & 59.7 & 49.8 & 56.7 & 50.6 & $-25.0$ \\
\midrule
{\ttfamily MeZO} & 93.7 & 69.0 & 56.6 & 60.5 & $-9.2$ \\
{\ttfamily SubZO} & 93.1 & 67.9 & 59.7 & 59.3 & $-9.2$ \\
{\ttfamily LOZO} & 93.6 & 69.5 & 59.6 & 60.2 & $-8.4$ \\
{\ttfamily TeZO} & 92.9 & 67.0 & 59.7 & 59.9 & $-9.2$ \\
\midrule
{\ttfamily MeZO-Adam} & 94.4 & 71.4 & 58.9 & 61.9 & $-7.5$ \\
{\ttfamily TeZO-Adam} & 94.2 & 75.0 & 58.9 & 60.8 & $-7.0$ \\
\midrule
\bottomrule
\end{tabular}}
\label{acc3}
\end{sc}
\end{center}
\vskip -0.15in
\end{table}

\textbf{Large-sized Models.} We conduct experiments on LLaMA-7B and OPT-13B, as shown in \textit{Table \ref{acc2}, \ref{acc3}}. Similarly, the reported results are the averages of 2 runs with different random seeds. On the large models, the low-rank ZO methods generally perform better than {\ttfamily MeZO} on OPT-13B and maintain similar performance on LLaMA-7B. The variants based on momentum and {\ttfamily Adam} perform better. {\ttfamily MeZO-m} and {\ttfamily MeZO-Adam} can achieve about 0.2\% and 2.1\% improvements. Due to the strong low-rank nature of {\ttfamily TeZO}, the alignment of factor vectors used in adaptivity still retains strong subspace properties. In practical training, the benefit of this advantage is that it constantly enforces the adaptive learning rate to stay synchronized with the structured subspace. Therefore, {\ttfamily TeZO-Adam} can achieve the better performance, about 2.2\% improvement on LLaMA-7B and 2.8\% improvement on OPT-13B compared to {\ttfamily MeZO}.

\textbf{Memory Usage and Wall-clock Time.} We evaluate the practical GPU memory usage and wall-clock time for different methods. \textit{Figure \ref{fg:memory and time}.} (a) shows the memory cost of ZO mainly consists of two parts, parameters and optimizer states. For the {\ttfamily MeZO} baseline, {\ttfamily -Adam} variant typically consumes 3$\times$ the storage. However, our proposed {\ttfamily TeZO-Adam} method requires less storage than {\ttfamily MeZO}, and is significantly lower than {\ttfamily MeZO-Adam} ($\sim$34.6\%). \textit{Figure \ref{fg:memory and time}.} (b) shows the wall-clock time comparisons, primarily including sampling, perturbations, forward pass, and update parameters. our {\ttfamily TeZO-Adam} maintains a speed comparable to the {\ttfamily MeZO} and is 1.63$\times$ faster than {\ttfamily MeZO-Adam} on one H100 device.

\begin{figure}[t]
\vskip -0.07in
\centering
    \subfigure[Training loss on SST-2.]{
	\includegraphics[width=0.24\textwidth]{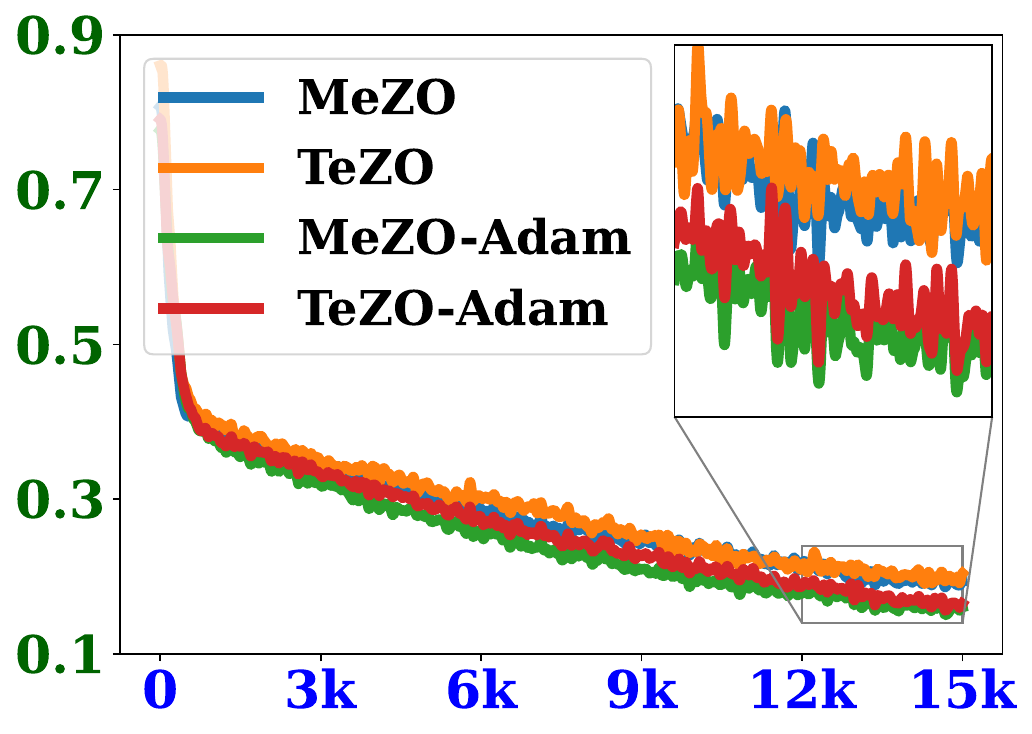}}\!\!\!\!
    \subfigure[Training loss on RTE.]{
	\includegraphics[width=0.24\textwidth]{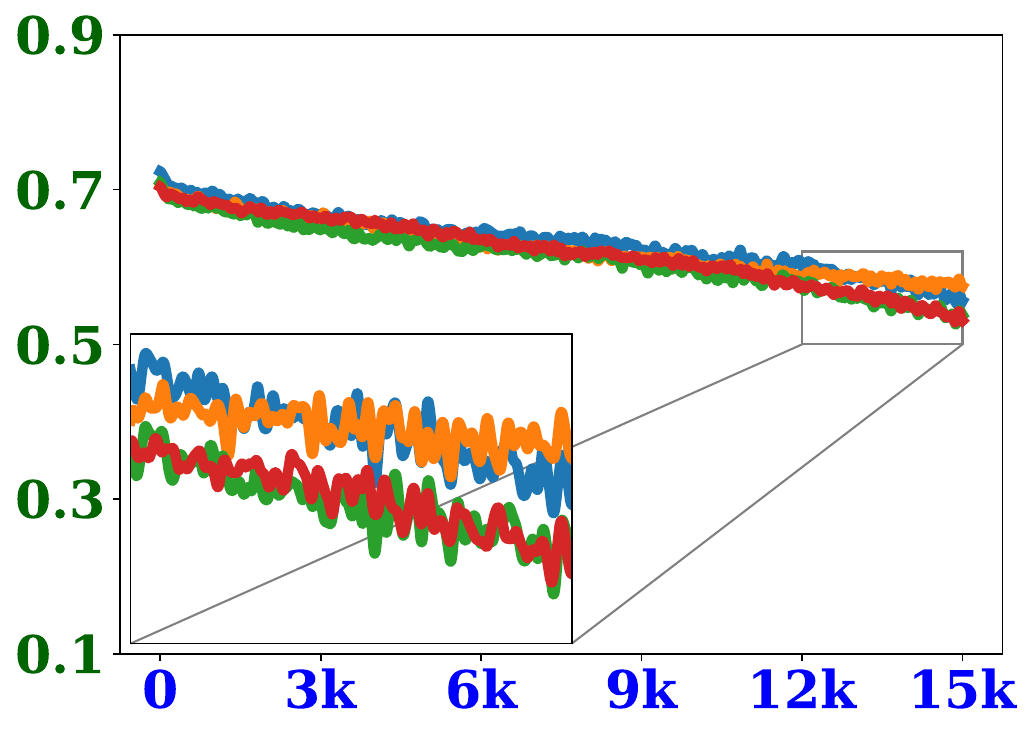}}
\vskip -0.15in
\caption{Loss curves of LLaMA-7B on SST-2 and RTE datasets on {\ttfamily ZO-SGD} and {\ttfamily ZO-Adam} methods. We use gaussian\_filter1d function in the scipy.ndimage lib to smooth curves with sigma=30.}
\label{fg:loss curves}
\vskip -0.1in
\end{figure}

\textbf{Better Performance in Adaptive ZO.} \textit{Figure \ref{fg:loss curves}.} illustrates the significant performance of adaptive ZO methods. It can be observed that the training loss curves of the {\ttfamily ZO-SGD} methods are nearly identical, indicating that these methods exhibit comparable performance. In contrast, the {\ttfamily ZO-Adam} method demonstrates superior performance, with a more pronounced reduction in loss curves and more thorough convergence during training, yielding better performance.
\section{Conclusion}
\label{tx:conclusion}
Inspired by the similarity in the gradient subspace, in this paper, we combine the low-rank properties in both the model and the temporal dimension and propose a novel low-rank ZO method, named {\ttfamily TeZO}. Moreover, {\ttfamily TeZO} can easily implement memory-efficient variants of momentum and {\ttfamily Adam}, maintaining the same resource consumption as standard {\ttfamily ZO-SGD}, but with better performance. We prove that {\ttfamily TeZO} maintains the same convergence rate as previous low-rank ZO methods while requiring fewer training costs. Furthermore, we conduct extensive evaluations of {\ttfamily TeZO} and its variants in fine-tuning tasks of LLMs, which demonstrates the significant potential of low-rank ZO methods.

\bibliography{example_paper}
\bibliographystyle{icml2025}

\newpage
\appendix
\onecolumn

\section{Experiment Materials}
\label{ap:experiments}

\subsection{Low Rankness in LLMs}
\label{ap:low_rank}
This property has been well studied and validated by several works. Especially in large models, the low-rank nature of parameters, gradients, and the FO optimizer states have triggered a series of studies. The most representative works include LoRA low-rank structure \cite{hu2021lora}, GaLore low-rank optimization \cite{zhao2024galore}, and so on. In the main text, we study the low-rankness on the OPT-1.3B model. Here, we also show tests of the low-rankness on the LLaMA-7B.

\subsubsection{Low Rankness of Each Single Gradient}
We first learn the low rankness of each single gradient.
Similarly, we consider the 2D parameters $W_l\in\mathbb{R}^{m\times n}$. Then we calculate the top-100 singular values of its gradients $\nabla_{W_l} f\in\mathbb{R}^{m\times n}$ to test the low-rankness, As shown in \textit{Figure \ref{ap:fg:llama low_rankness}}.
\begin{figure}[H]
\centering
    \subfigure[{\textnormal{\ttfamily layers.6.self\_attn.k\_proj}}]{
	\includegraphics[width=0.49\textwidth]{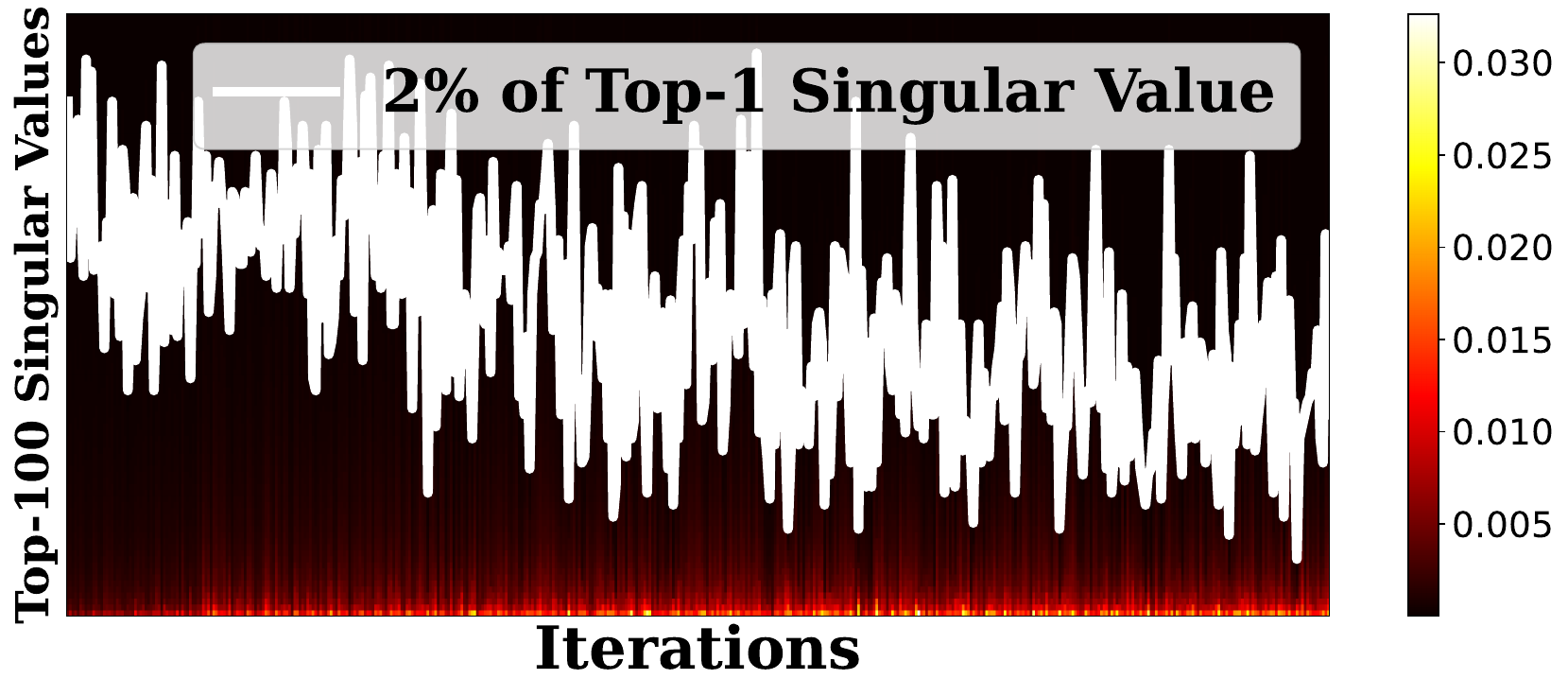}}
    \subfigure[{\textnormal{\ttfamily layers.12.self\_attn.v\_proj}}]{
	\includegraphics[width=0.49\textwidth]{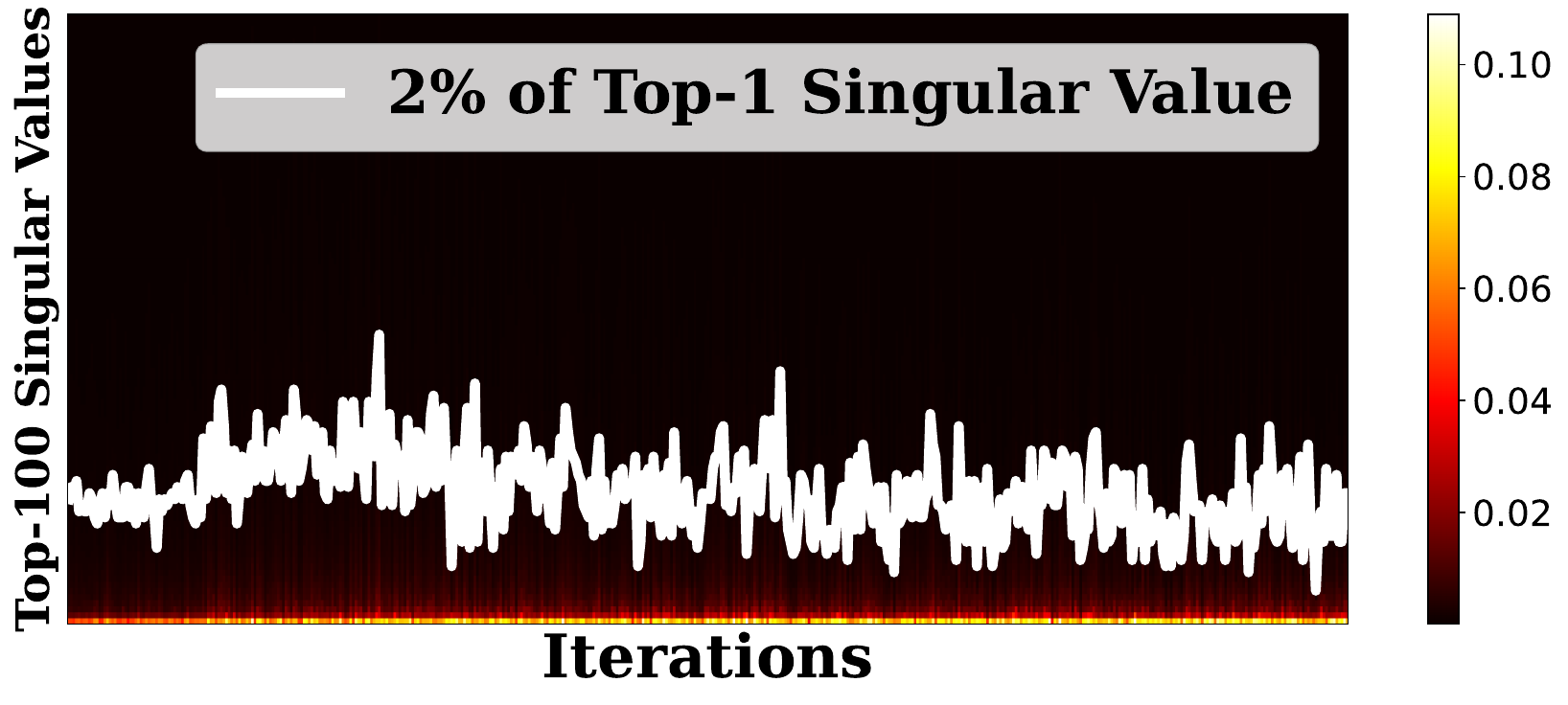}}
    \subfigure[{\textnormal{\ttfamily layers.18.self\_attn.q\_proj}}]{
	\includegraphics[width=0.49\textwidth]{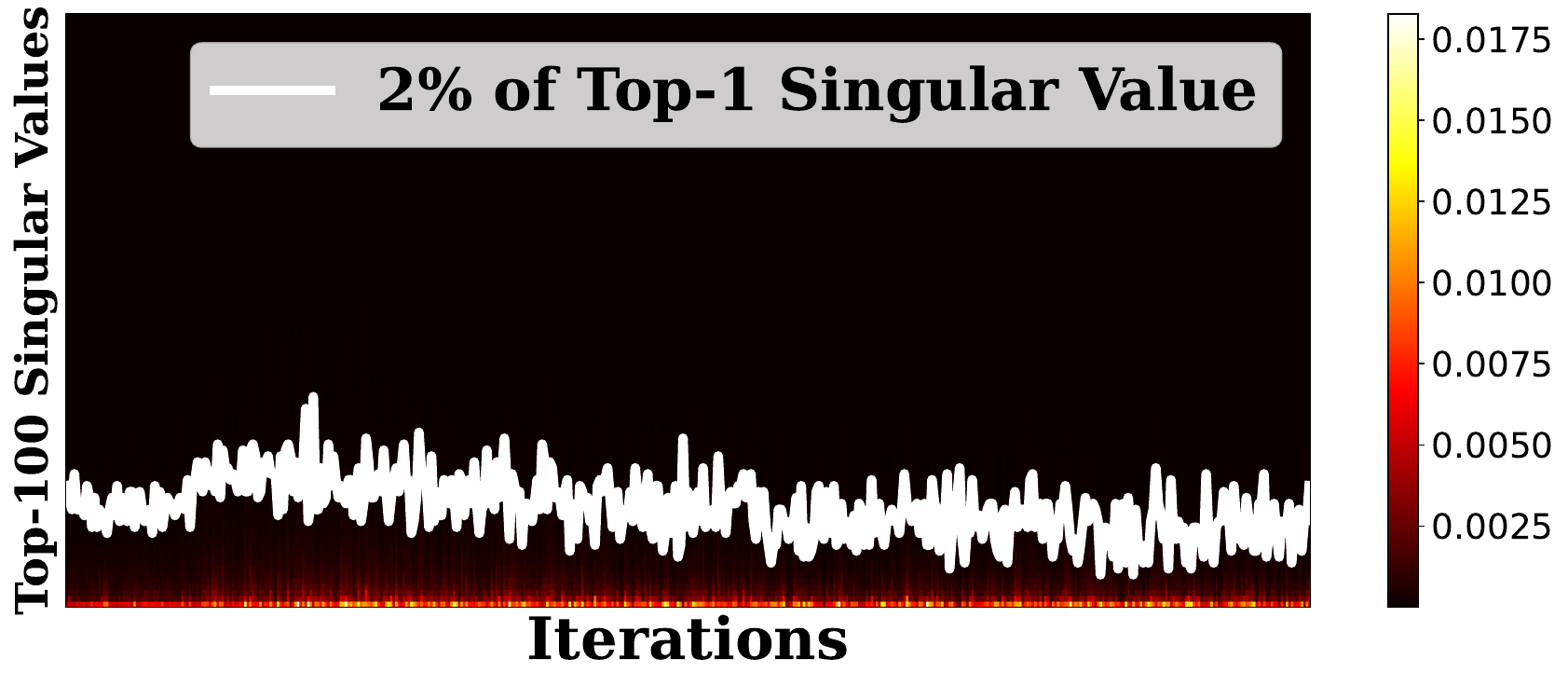}}
    \subfigure[{\textnormal{\ttfamily layers.24.self\_attn.o\_proj}}]{
	\includegraphics[width=0.49\textwidth]{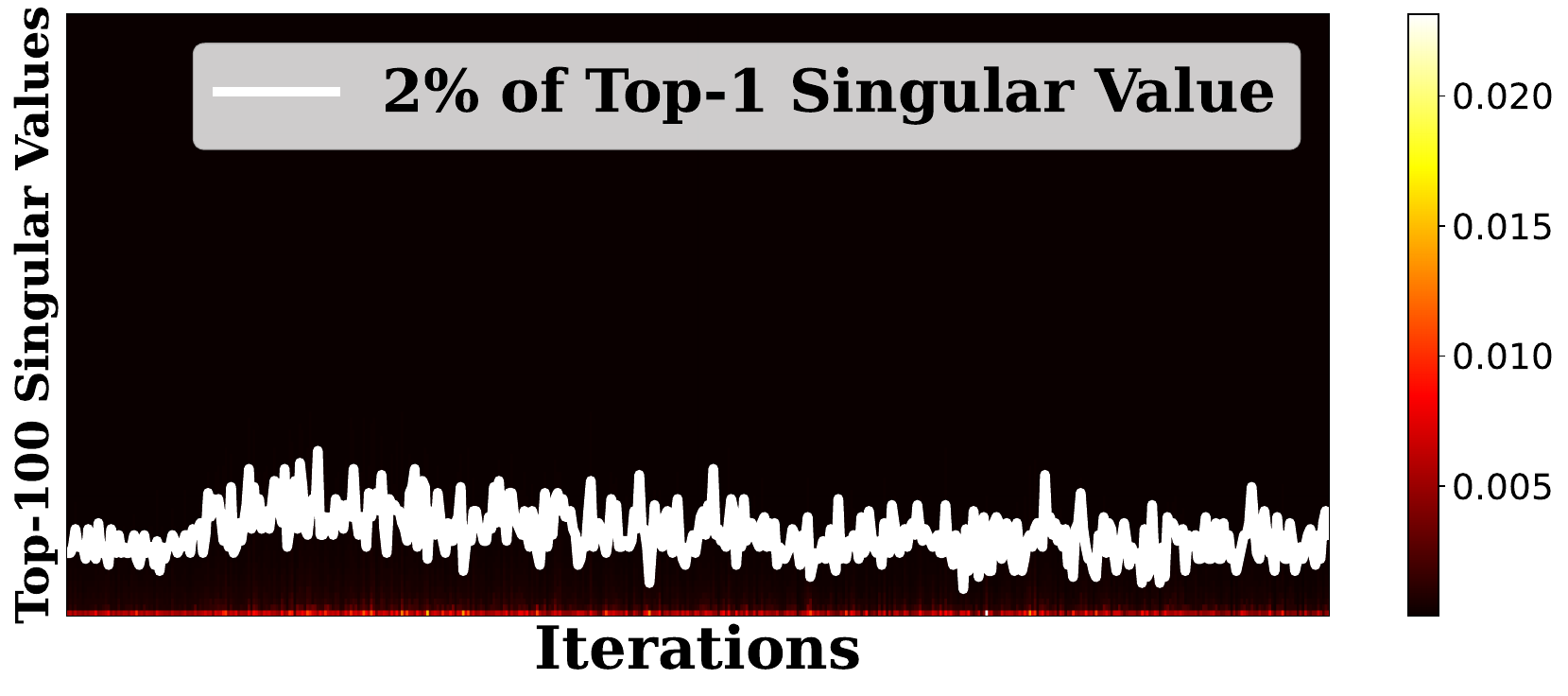}}
\caption{We finetune LLaMA-7B on SST-2 to test low-rankness of gradients. We set the batchsize as 16 and train 500 steps with 8000 samples on a H200 device. The training loss decreases from 1.04 to 0.13. We analyze the low-rank properties of the $W_K$, $W_V$, $W_Q$ and $W_O$ parameters in the $6$-th, $12$-th, $18$-th, and $24$-th modules at each iteration ($W_K,W_V,W_Q,W_O\in\mathbb{R}^{4096\times4096}$). The white lines represent the indices where the singular values are 2\% of the maximum singular value.}
\label{ap:fg:llama low_rankness}
\end{figure}
It is clear that gradients are low-rank on LLaMA-7B model and the low-rankness is even greater than that of OPT-1.3B. After around index-20, the singular value is almost completely lost. It is worth noting that in our tests, \textbf{the data samples used for each gradient computation are completely different}, which further emphasizes the universality of gradient low-rankness in LLMs. 

\subsubsection{Low Rankness of Gradient Subspace}
The low-rankness of each individual gradient has already been widely acknowledged. In this part, we continue to explore the low-rank subspace of all gradients in the LLaMA-7B model. The same, we consider the 2D parameters $W_l\in\mathbb{R}^{m\times n}$ trained for $T$ iterations. We normalize each flattened gradient and concatenate them along the $T$ dimension to form a new matrix as $G = \left[g_{w_l,0}, g_{w_l,1}, \cdots, g_{w_l,T}\right]\in\mathbb{R}^{mn\times T}$ where $g_{w_l,t}=\nabla_{w_l}f_t/\Vert\nabla_{w_l}f_t\Vert\in\mathbb{R}^{mn}$. And then we calculate the cosine value by $G^\top G$. It is important to note that \textbf{without normalization, the low-rankness of this matrix naturally holds.} This is because when the loss is large, the gradients are naturally large. As training progresses and the loss becomes smaller, the gradients will be much smaller. If these gradients are concatenated directly, although it still forms a low-rank matrix, this low-rank nature is inconsistent with the motivation behind our proposed {\ttfamily TeZO} method. {\ttfamily TeZO} expects similarity across the entire gradient space. Since we are always more concerned with whether the gradient direction is similar, we study the properties of each normalized gradient, specifically whether all gradients lie in the same subspace, as shown in \textit{Figure \ref{ap:fg:llama low_rank subspace}}.

\begin{figure}[H]
\centering
    \subfigure[{\textnormal{\ttfamily layers.15.self\_attn.v\_proj}}]{
	\includegraphics[width=0.35\textwidth]{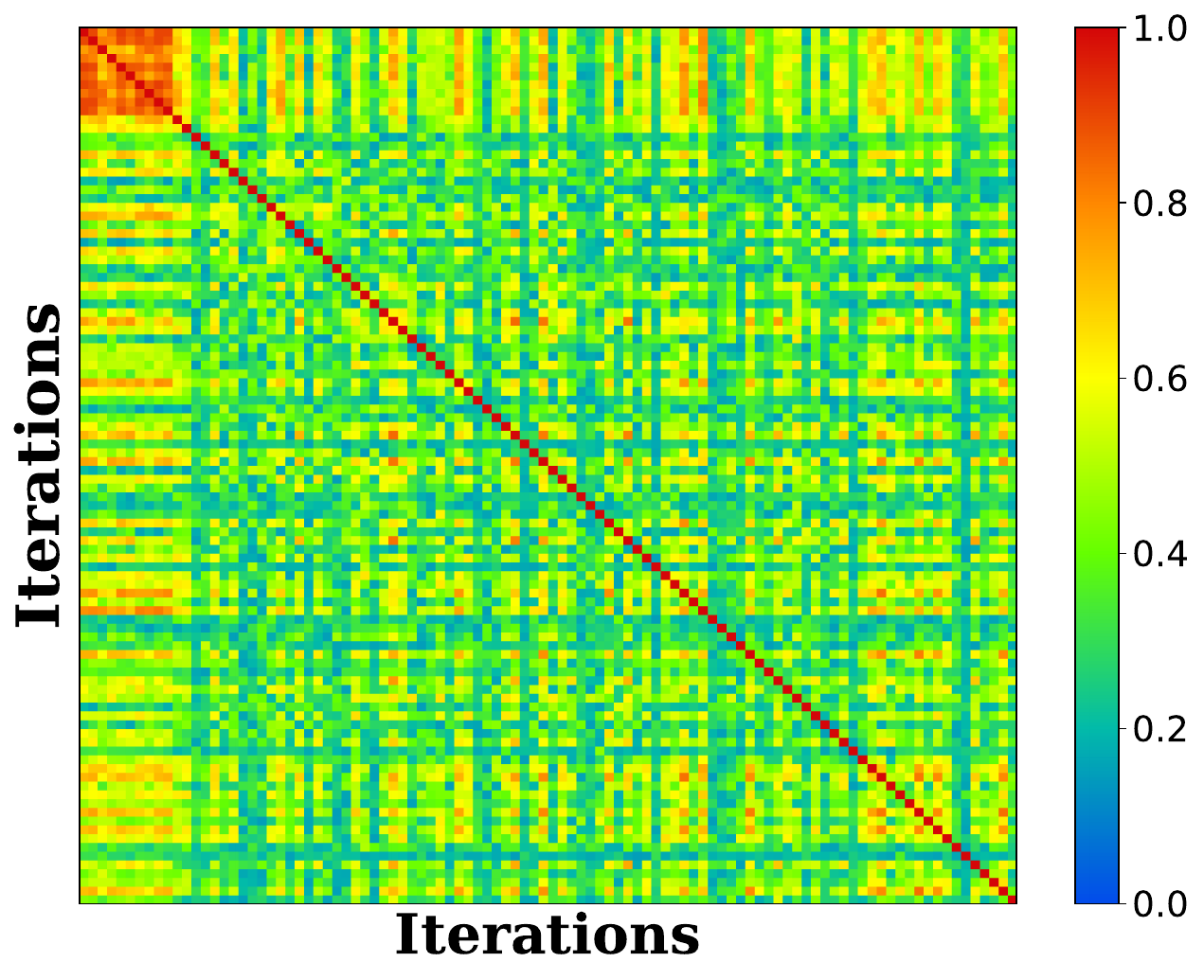}}
    \subfigure[{\textnormal{\ttfamily layers.28.self\_attn.o\_proj}}]{
	\includegraphics[width=0.35\textwidth]{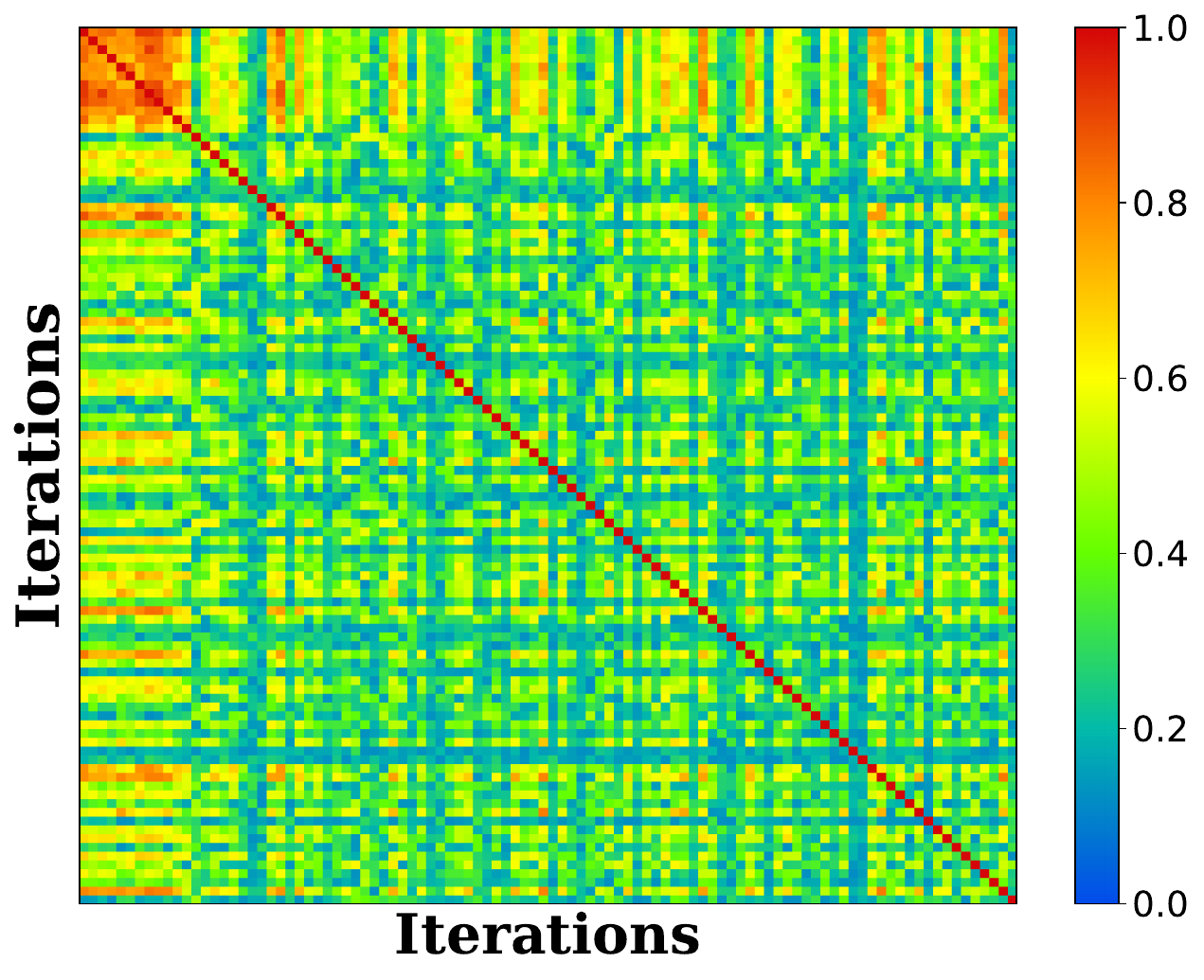}}
\caption{We finetune LLaMA-7B on SST-2 to test the similarity between gradients at different iterations. Similarly, we set the batchsize as 16 and train 500 steps with 8000 samples on a H200 device. The training loss decreases from 1.04 to 0.13. We calculate the cosine value of each gradient pair $(\nabla_{W_l} f_{t_1}, \nabla_{W_l} f_{t_2})$ where $t_1, t_2\in\left[0,1,2,\cdots,499\right]$ and show their values as the heat maps above.}
\label{ap:fg:llama low_rank subspace}
\end{figure}

It can be seen that the similarity between gradients is generally high, and the distribution of cosine distances is relatively concentrated. This also highlights the low-rankness of the gradient space in training LLMs, where gradients from different samples exhibit strong similarity.

\subsubsection{Low-rankness between Weights and Gradients}
In this part, we explore the close relationship between the low-rankness of gradients and that of model parameters. 

\begin{figure}[H]
\centering
    \subfigure[{\textnormal{\ttfamily layers.6.self\_attn.k\_proj}}]{
	\includegraphics[width=0.49\textwidth]{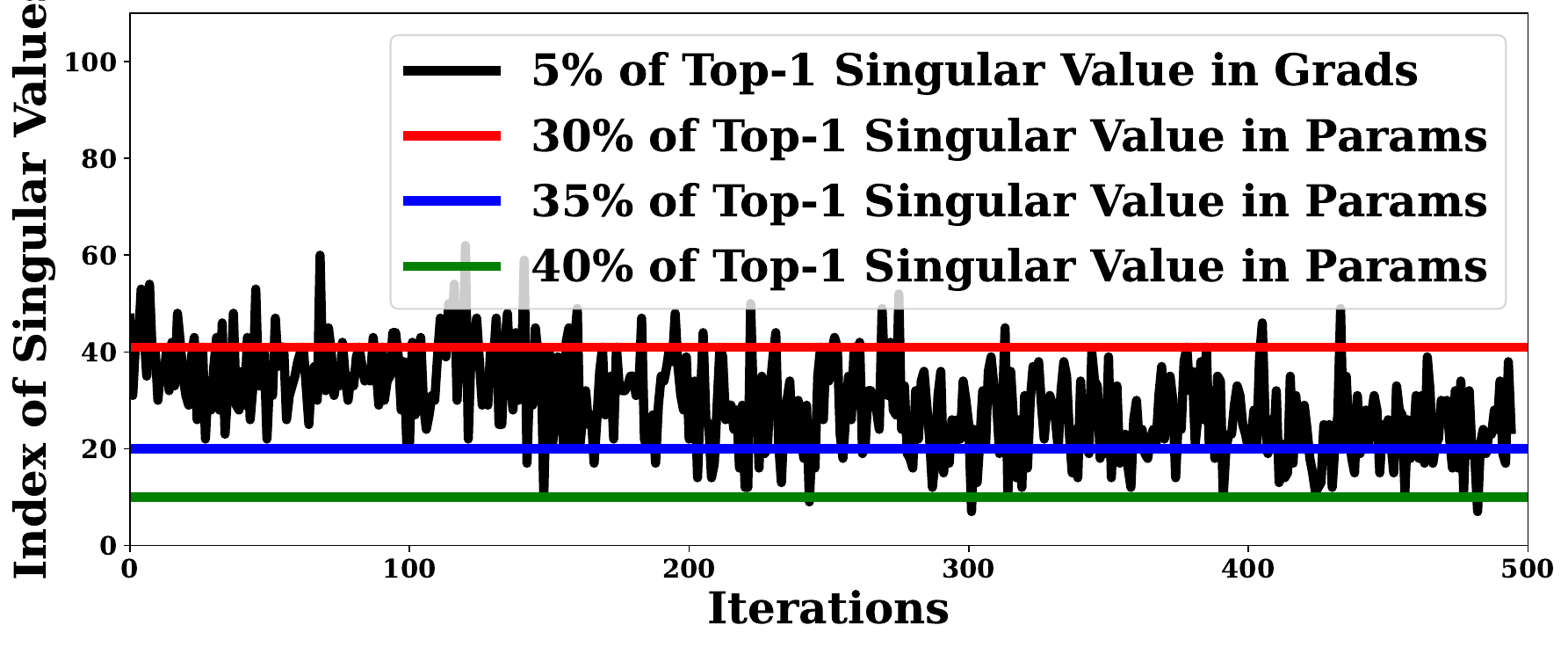}}
    \subfigure[{\textnormal{\ttfamily layers.18.self\_attn.q\_proj}}]{
	\includegraphics[width=0.49\textwidth]{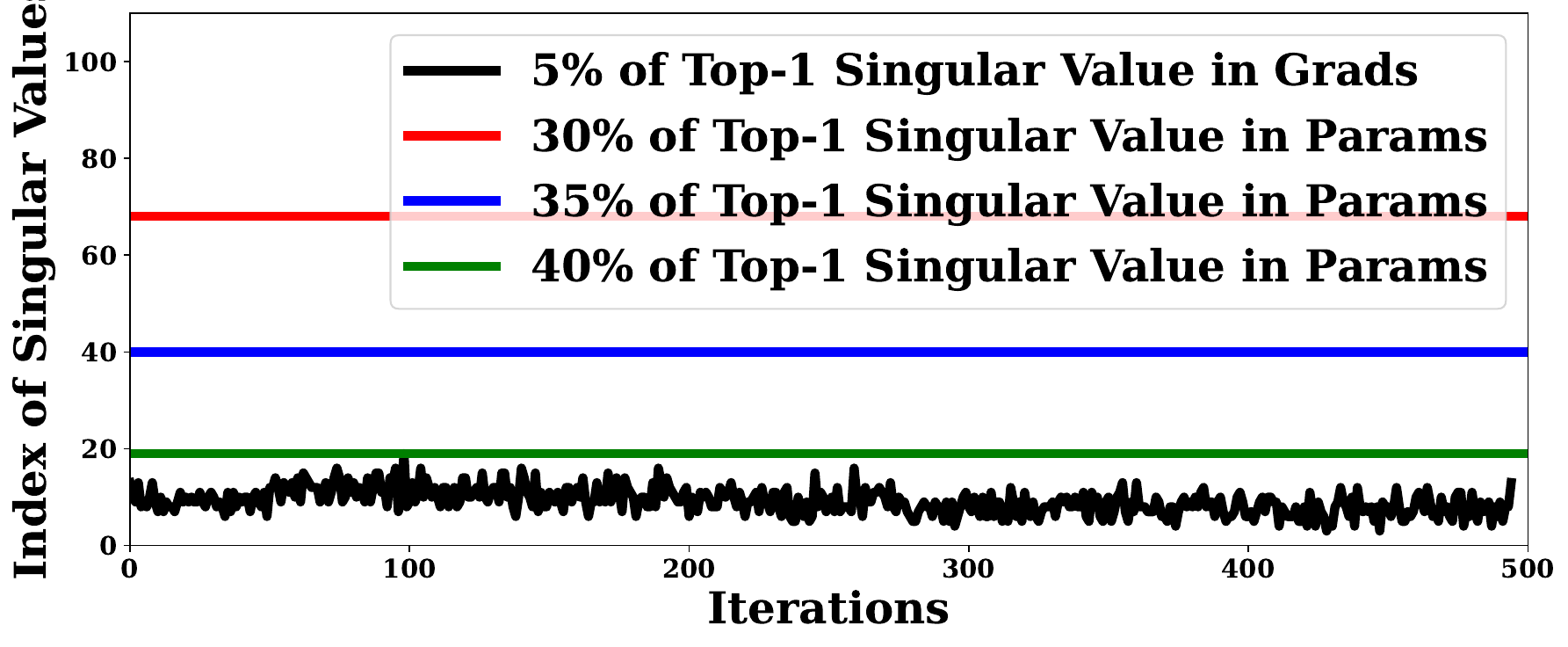}}
\caption{We finetune LLaMA-7B on SST-2 to test the similarity between gradients at different iterations. We compared the relationship between the low-rank properties of parameters and gradients, and demonstrat the rank levels of parameters.}
\label{ap:fg:llama low_rank weight vs grad}
\end{figure}

We can observe that, although the degree of low-rankness in gradients and parameters is not strictly aligned, they still exhibit a high degree of correlation within a certain fluctuation range. In fact, due to the very small learning rate, the gradient updates for individual parameters are negligible, which allows the low-rank nature of the model to remain relatively stable. This also validates the effectiveness of our dynamic selection method within a certain range. The dynamic selection method eliminates the need for additional hyperparameter tuning while ensuring experimental stability.

\newpage
\subsection{The Efficiency of Lightweight Second-order Momentum in {\ttfamily TeZO-Adam}}
\label{ap:tezo-adam}
In this paper, we propose a lightweight variant to address the storage issue of second-order momentum in the {\ttfamily TeZO-Adam} variant:
\begin{equation}
    \left[\nabla^0 f(w_t)\right]^2 =
    \kappa_t^2Z_t^2 = \kappa_t^2\left(\sum_{s=1}^{r_l}\tau_s\cdot(u_s\circ v_s)\right)^2 = \underbrace{\sum_{s=1}^{r_l}\kappa_t^2\tau_s^2\cdot(u_s^2\circ v_s^2)}_{\text{Separable Term}} + \underbrace{\kappa_t^2\sum_{p\neq q}^{r_l} \tau_p\tau_q\cdot(u_p u_q\circ v_p v_q)}_{\mathbb{E}_{\tau,u,v}\left[\tau_p\tau_q\cdot(u_p u_q\circ v_p v_q) \right]={\bm 0}}.
\end{equation}
The separable term is memory-efficient which can be calculated by the accumulation of the factor vector $\tau$.

\subsubsection{Errors in One Step}
By the definitions, we consider the decomposition of $Z\in\mathbb{R}^{m\times n}$ by the factor vectors $u_s\in\mathbb{R}^{m}$, $v_s\in\mathbb{R}^{n}$, and $\tau\in\mathbb{R}^r$. And we consider an example comparable in scale to the LLaMA-7B model and set $m=n=4096$. And we select $r=64$ to evaluate the error. Since we consider the parameters at time $t$ where $\kappa_t$ can be treated as a constant. Without loss of generality, we set $\kappa_t=1$ directly and examine the error by randomly sampling $\tau,u,v$, as shown in the following figure.
\begin{align*}
\underbrace{\raisebox{-0.5\height}{\includegraphics[width=0.25\textwidth]{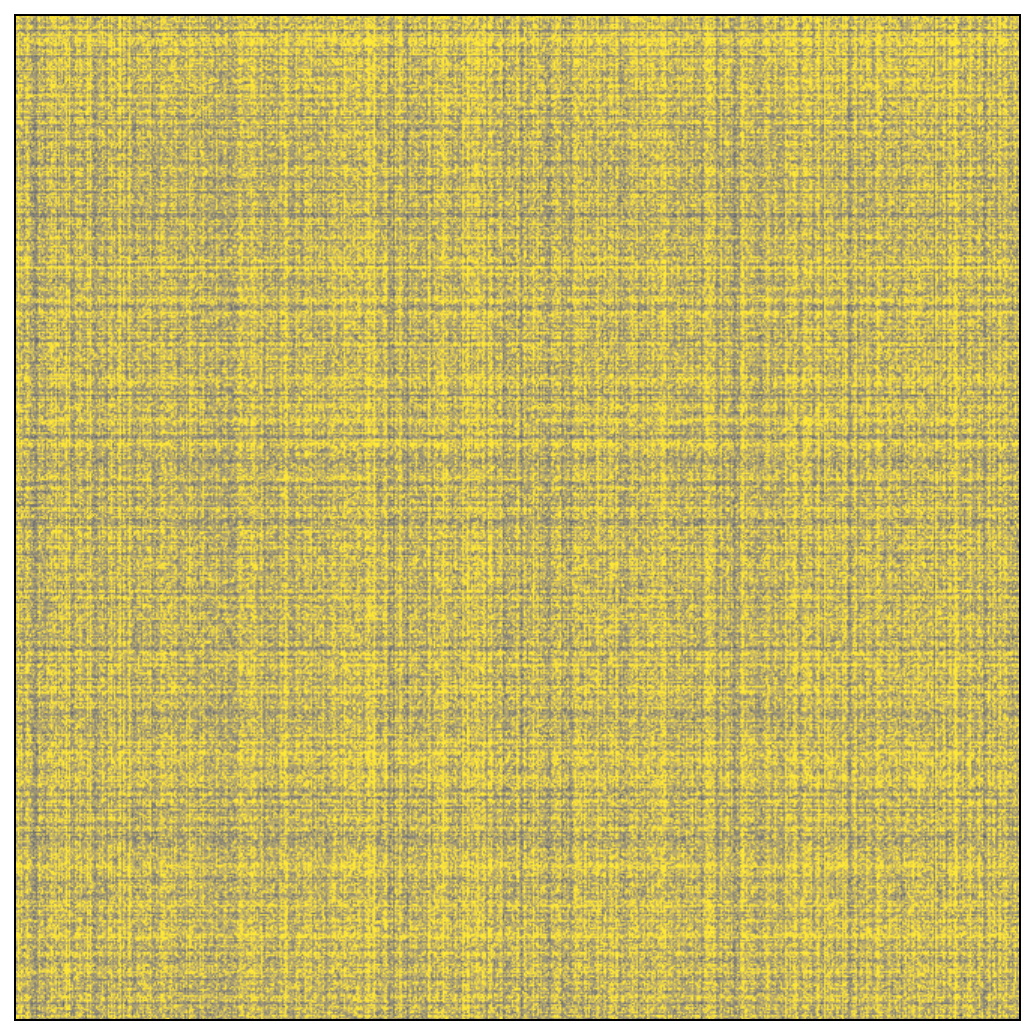}}}_{\left(\sum_{s=1}^{r_l}\tau_s\cdot(u_s\circ v_s)\right)^2} = \underbrace{\raisebox{-0.5\height}{\includegraphics[width=0.25\textwidth]{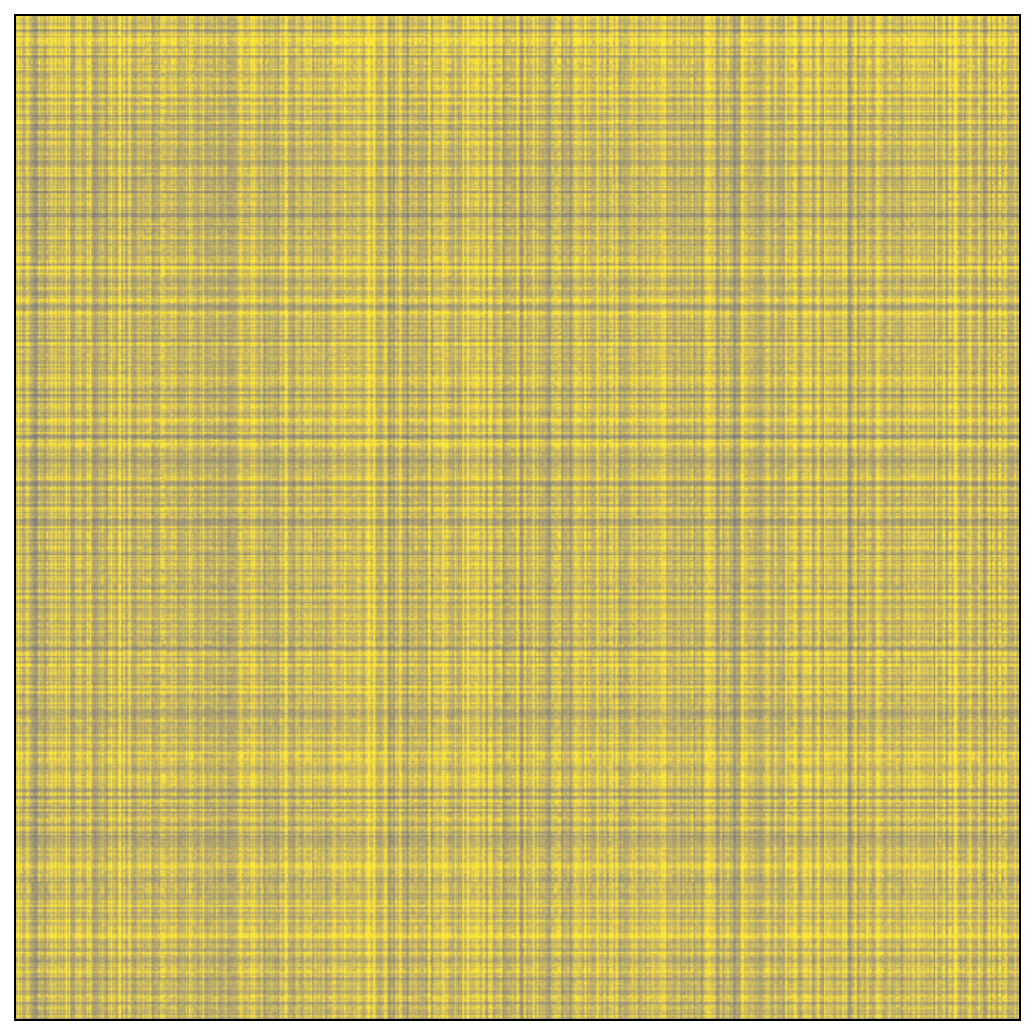}}}_{\sum_{s=1}^{r}\tau_s^2\cdot(u_s^2\circ v_s^2)} + \underbrace{\raisebox{-0.512\height}{\includegraphics[width=0.315\textwidth]{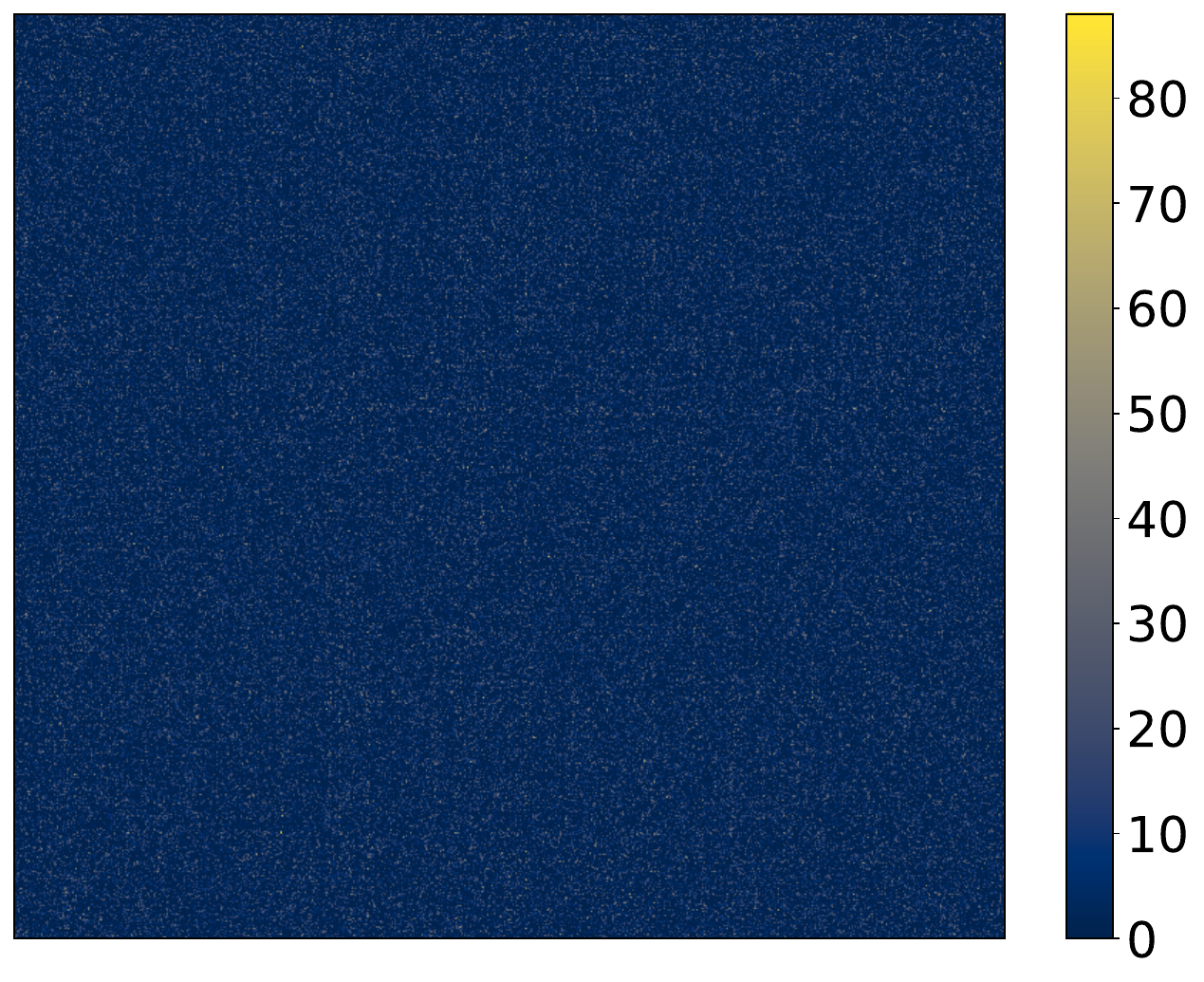}}}_{\sum_{p\neq q}^{r} \tau_p\tau_q\cdot(u_p u_q\circ v_p v_q)\approx \bf{0}}
\end{align*}
This clearly demonstrates the precision of our lightweight estimation. The second term is almost zero, and the cost of calculating it is very high, including both storage and computation. Therefore, we eliminate the second term directly and accumulate the second-order momentum of the first term as the second-order momentum of {\ttfamily Adam} for updates. This significantly reduces the training cost, making the training overhead of our {\ttfamily TeZO-Adam} method almost consistent with that of {\ttfamily MeZO-SGD}, significantly lower than the {\ttfamily MeZO-Adam} method.

\subsubsection{Accumulated Errors after $T$ Steps}
Then we learn the accumulated errors in the training process. We define the update of standard second-order momentum as $V_{t+1} = \beta_2 V_t + (1-\beta_2)\left(\sum_{s=1}^{r}\tau_{s,t}\cdot(u_s\circ v_s)\right)^2$, and that of {\ttfamily TeZO-Adam} as $\hat{V}_{t+1} = \beta_2 \hat{V}_{t} + (1-\beta_2)\sum_{s=1}^{r}\tau_{s,t}^2\cdot(u_s^2\circ v_s^2)$. We report the averaged accumulated errors $E_t = (V_{t} - \hat{V}_t)/mn$ over 1000 steps under $\beta_2=0.99$, as shown in \textit{Figure \ref{ap:fg:accumulated errors}}.

\begin{figure}[H]
\centering
\includegraphics[width=0.52\textwidth]{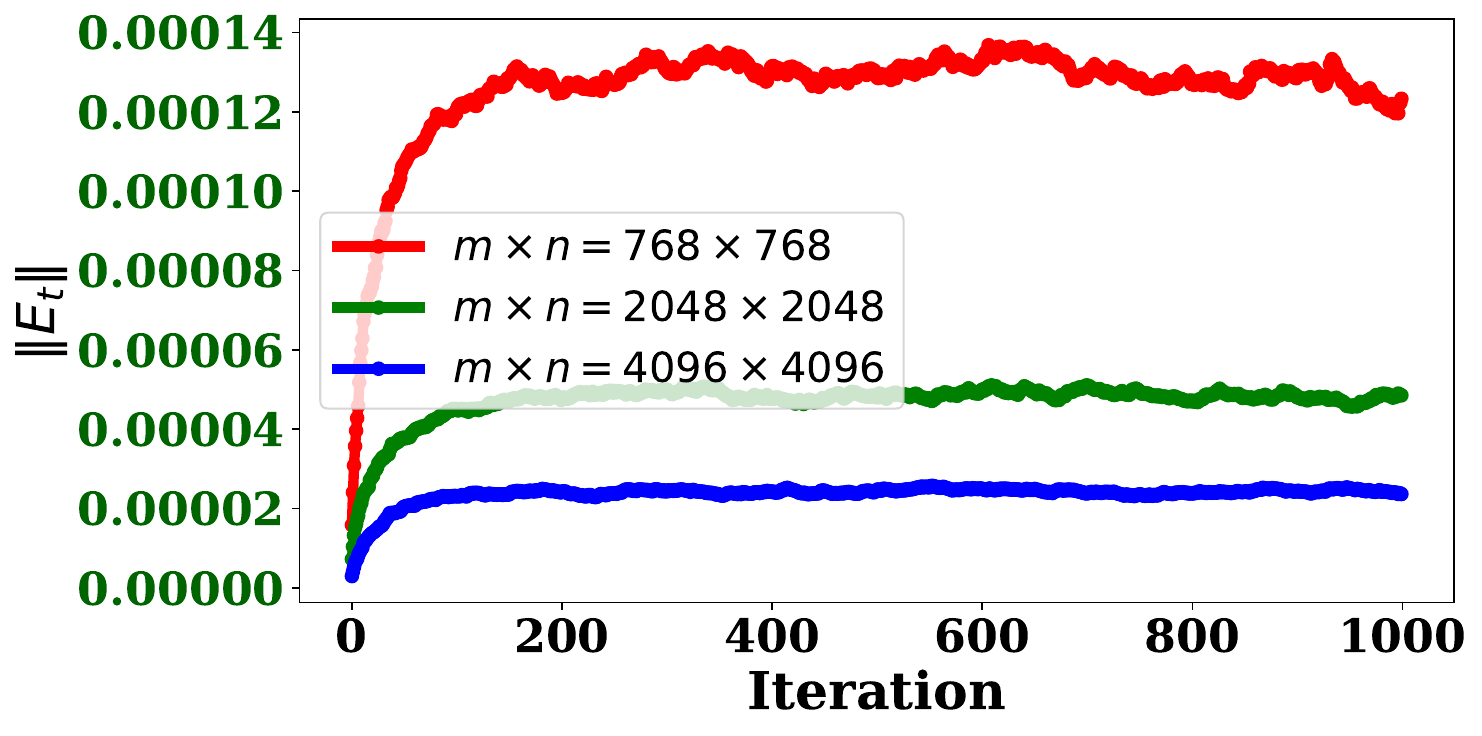}
\caption{$\Vert E_t\Vert$ under different $m, n$ and $r=64$. It can be observed that the averaged accumulated errors decrease as the model size increases, which highlights the practicality of our proposed lightweight second-order moment estimation on LLMs.}
\label{ap:fg:accumulated errors}
\end{figure}

\newpage
\subsection{Setups and Hyperparameters}
\label{ap:hyperparameters}

We follow previous works \cite{malladi2023fine,chen2024enhancing,yu2024subzero} and summarized the range of hyperparameter selections, as shown in \textit{Table \ref{ap:tb:hyperparamters}}. Although certain hyperparameters, such as batchsize and perturbation rate, may introduce subtle variations, we fix these hyperparameters across all methods to ensure fairness. The primary search hyperparameter is the learning rate across models of different scales.

\begin{table}[H]
\centering
\vskip -0.1in
\caption{Hyperparameter recommendations for different models.}
\vskip 0.05in
\label{ap:tb:hyperparamters}
\begin{tabular}{c|cc|c|c|c}
\toprule
\midrule
 Method & Hyperparameters & Search Range & RoBERTa-large & OPT-13B & LLaMA-7B \\ 
 \midrule
 \multirow{3.5}{*}{\makecell{{\ttfamily MeZO} \\ {\ttfamily MeZO-m}}} & Batchsize &  \{16,32,64\} & 64 & \multicolumn{2}{c}{16} \\
 \cmidrule(lr){2-6}
                  & learning rate & \{1e-4, 1e-5, 1e-6, 1e-7\} & 1e-6 & 1e-7 & 1e-6 \\
 \cmidrule(lr){2-6}
                  & perturbation rate & 1e-3 & \multicolumn{3}{c}{1e-3} \\
 \midrule
 \multirow{3.5}{*}{{\ttfamily MeZO-Adam}} & Batchsize & \{16,32,64\} & 64 & \multicolumn{2}{c}{16} \\
 \cmidrule(lr){2-6}
                  & learning rate & \{1e-4, 3e-5, 1e-5, 3e-6\} & - & 1e-5 & 3e-5 / 1e-5 \\
 \cmidrule(lr){2-6}
                  & perturbation rate & 1e-3 & \multicolumn{3}{c}{1e-3} \\
 \midrule
 \multirow{6.5}{*}{{\ttfamily SubZO}} & Batchsize & \{16,32,64\} & 64 & \multicolumn{2}{c}{16} \\
 \cmidrule(lr){2-6}
                   & learning rate & \{1e-4, 1e-5, 1e-6, 1e-7\} & 1e-6 & 1e-7 & 1e-6 \\
 \cmidrule(lr){2-6}
                   & perturbation rate & 1e-3 & \multicolumn{3}{c}{1e-3} \\
 \cmidrule(lr){2-6}
                   & rank & \{32, 64, 128\} & \multicolumn{3}{c}{64} \\
 \cmidrule(lr){2-6}
                   & lazy update interval & \{50, 100, 500\} & \multicolumn{3}{c}{500} \\
 \midrule
 \multirow{6.5}{*}{\makecell{{\ttfamily LOZO} \\ {\ttfamily LOZO-m}}} & Batchsize & \{16,32,64\} & 64 & \multicolumn{2}{c}{16} \\
 \cmidrule(lr){2-6}
                  & learning rate & \{1e-4, 1e-5, 1e-6, 1e-7\} & 1e-6 & 1e-7 & 1e-6 \\
 \cmidrule(lr){2-6}
                  & perturbation rate & 1e-3 & \multicolumn{3}{c}{1e-3} \\
 \cmidrule(lr){2-6}
                  & rank & \{8, 16, 32\} & \multicolumn{3}{c}{8} \\
 \cmidrule(lr){2-6}
                   & lazy update interval & \{50, 100, 500\} & \multicolumn{3}{c}{100} \\
 \midrule
 \multirow{6.5}{*}{\makecell{{\ttfamily TeZO} \\ {\ttfamily TeZO-m}}} & Batchsize & \{16,32,64\} & 64 & \multicolumn{2}{c}{16} \\
 \cmidrule(lr){2-6}
                  & learning rate & \{1e-4, 1e-5, 1e-6, 1e-7\} & 1e-6 & 1e-7 & 1e-6 \\
 \cmidrule(lr){2-6}
                  & perturbation rate & 1e-3 & \multicolumn{3}{c}{1e-3} \\
 \cmidrule(lr){2-6}
                  & threshold to select rank & \{20\%, 25\%, 30\%, 35\%\} & \multicolumn{3}{c}{25\% / 30\%} \\
 \cmidrule(lr){2-6}
                  & maximum threshold of rank & \{32, 64, 128, 256\} & \multicolumn{3}{c}{depend on tasks} \\
 \midrule
 \multirow{6.5}{*}{{\ttfamily TeZO-Adam}} & Batchsize & \{16,32,64\} & 64 & \multicolumn{2}{c}{16} \\
 \cmidrule(lr){2-6}
                       & learning rate & \{1e-4, 3e-5, 1e-5, 3e-6\} & - & 1e-5 & 3e-5 / 1e-5 \\
 \cmidrule(lr){2-6}
                       & perturbation rate & 1e-3 & \multicolumn{3}{c}{1e-3} \\
 \cmidrule(lr){2-6}
                       & threshold to select rank & \{20\%, 25\%, 30\%, 35\%\} & \multicolumn{3}{c}{25\% / 30\%} \\
 \cmidrule(lr){2-6}
                       & maximum threshold of rank & \{32, 64, 128, 256\} & \multicolumn{3}{c}{depend on tasks} \\
 \midrule
\bottomrule
\end{tabular}
\end{table}
We refer to the selections reported in previous works and grid search each hyperparameter. Although further fine-tuning of hyperparameters for specific tasks could yield greater benefits, we fix the hyperparameter selections for fairness. The recommended value reported in the table above is provided only as a reference on which most tasks work well.

\newpage
\subsection{Memory Usage and Wall-clock Time on Different Model Sizes}
\label{ap:memory and time}
We extensively test the training efficiency of the OPT and LLaMA models across different model sizes, as shown in \textit{Table \ref{ap:tb:memory}} and \textit{Table \ref{ap:tb:time}}. 
\begin{table}[H]
\centering
\vskip -0.1in
\caption{GPU memory usage (max memory reserved) for fine-tuning LLMs on RTE dataset on a single H100 device.}
\vskip 0.05in
\label{ap:tb:memory}
\begin{tabular}{c|cccccc|ccc}
\toprule
\midrule
\multicolumn{1}{c}{} & \multicolumn{6}{c}{OPT} & \multicolumn{3}{c}{LLaMA} \\
\cmidrule(lr){2-7} \cmidrule(lr){8-10}
 \multicolumn{1}{c}{} & 125M & 1.3B & 2.7B & 6.7B & 13B & 30B & 7B & 13B & 30B \\
\midrule
{\ttfamily Zero-Shot} & 0.36 {\ttfamily G} & 2.90 {\ttfamily G} & 5.44 {\ttfamily G} & 12.73 {\ttfamily G} & 24.39 {\ttfamily G} & 56.46 {\ttfamily G} & 12.92 {\ttfamily G} & 24.89 {\ttfamily G} & 61.86 {\ttfamily G} \\
\midrule
{\ttfamily MeZO} & 0.57 {\ttfamily G} & 3.48 {\ttfamily G} & 6.40 {\ttfamily G} & 14.40 {\ttfamily G} & 26.43 {\ttfamily G} & 60.31 {\ttfamily G} & 13.91 {\ttfamily G} & 26.12 {\ttfamily G} & 63.77 {\ttfamily G} \\
{\ttfamily SubZO} & 0.54 {\ttfamily G} & 3.28 {\ttfamily G} & 5.92 {\ttfamily G} & 14.91 {\ttfamily G} & 26.97 {\ttfamily G} & 61.18 {\ttfamily G} & 14.28 {\ttfamily G} & 26.67 {\ttfamily G} & 64.45 {\ttfamily G} \\
{\ttfamily LOZO} & 0.52 {\ttfamily G} & 3.31 {\ttfamily G} & 5.95 {\ttfamily G} & 13.66 {\ttfamily G} & 25.50 {\ttfamily G} & 57.93 {\ttfamily G} & 13.44 {\ttfamily G} & 25.77 {\ttfamily G} & 62.38 {\ttfamily G} \\
{\ttfamily TeZO} & 0.54 {\ttfamily G} & 3.28 {\ttfamily G} & 5.92 {\ttfamily G} & 13.68 {\ttfamily G} & 25.52 {\ttfamily G} & 57.95 {\ttfamily G} & 13.47 {\ttfamily G} & 25.79 {\ttfamily G} & 62.40 {\ttfamily G} \\
\midrule
{\ttfamily MeZO-m} & 0.89 {\ttfamily G} & 6.31 {\ttfamily G} & 11.77 {\ttfamily G} & 27.19 {\ttfamily G} & 51.32 {\ttfamily G} & $>$80 {\ttfamily G} & 26.85 {\ttfamily G} & 51.31 {\ttfamily G} & $>$80 {\ttfamily G} \\
{\ttfamily LOZO-m} & 0.53 {\ttfamily G} & 3.32 {\ttfamily G} & 5.97 {\ttfamily G} & 13.68 {\ttfamily G} & 25.53 {\ttfamily G} & 57.99 {\ttfamily G} & 13.47 {\ttfamily G} & 25.80 {\ttfamily G} & 62.44 {\ttfamily G} \\
{\ttfamily TeZO-m} & 0.55 {\ttfamily G} & 3.29 {\ttfamily G} & 5.93 {\ttfamily G} & 13.69 {\ttfamily G} & 25.52 {\ttfamily G} & 57.96 {\ttfamily G} & 13.48 {\ttfamily G} & 25.79 {\ttfamily G} & 62.40 {\ttfamily G} \\
\midrule
{\ttfamily MeZO-Adam} & 1.19 {\ttfamily G} & 9.15 {\ttfamily G} & 16.90 {\ttfamily G} & 39.98 {\ttfamily G} & 75.27 {\ttfamily G} & $>$80 {\ttfamily G} & 39.50 {\ttfamily G} & 75.69 {\ttfamily G} & $>$80 {\ttfamily G} \\
{\ttfamily TeZO-Adam} & 0.57 {\ttfamily G} & 3.48 {\ttfamily G} & 6.16 {\ttfamily G} & 14.07 {\ttfamily G} & 26.01 {\ttfamily G} & 58.64 {\ttfamily G} & 13.71 {\ttfamily G} & 26.16 {\ttfamily G} & 62.80 {\ttfamily G} \\
\bottomrule
\end{tabular}
\vskip -0.1in
\end{table}

\begin{table}[H]
\centering
\vskip -0.1in
\caption{Wall-clock time per iteration for fine-tuning LLMs on RTE dataset on a single H100 device.}
\vskip 0.05in
\label{ap:tb:time}
\begin{tabular}{c|cccccc|ccc}
\toprule
\midrule
\multicolumn{1}{c}{} & \multicolumn{6}{c}{OPT} & \multicolumn{3}{c}{LLaMA} \\
\cmidrule(lr){2-7} \cmidrule(lr){8-10}
 \multicolumn{1}{c}{} & 125M & 1.3B & 2.7B & 6.7B & 13B & 30B & 7B & 13B & 30B \\
\midrule
{\ttfamily Zero-Shot} & - & - & - & - & - & - & - & - & - \\
\midrule
{\ttfamily MeZO} & 33 {\ttfamily ms} & 69 {\ttfamily ms} & 111 {\ttfamily ms} & 212 {\ttfamily ms} & 388 {\ttfamily ms} & 871 {\ttfamily ms} & 212 {\ttfamily ms} & 372 {\ttfamily ms} & 942 {\ttfamily ms} \\
{\ttfamily SubZO} & 39 {\ttfamily ms} & 75 {\ttfamily ms} & 121 {\ttfamily ms} & 211 {\ttfamily ms} & 373 {\ttfamily ms} & 939 {\ttfamily ms} & 218 {\ttfamily ms} & 385 {\ttfamily ms} & 988 {\ttfamily ms}\\
{\ttfamily LOZO} & 34 {\ttfamily ms} & 71 {\ttfamily ms} & 109 {\ttfamily ms} & 191 {\ttfamily ms} & 341 {\ttfamily ms} & 745 {\ttfamily ms} & 195 {\ttfamily ms} & 350 {\ttfamily ms} & 832 {\ttfamily ms}\\
{\ttfamily TeZO} & 36 {\ttfamily ms} & 67 {\ttfamily ms} & 106 {\ttfamily ms} & 178 {\ttfamily ms} & 316 {\ttfamily ms} & 680 {\ttfamily ms} & 186 {\ttfamily ms} & 325 {\ttfamily ms} & 775 {\ttfamily ms}\\
\midrule
{\ttfamily MeZO-m} & 32 {\ttfamily ms} & 76 {\ttfamily ms} & 123 {\ttfamily ms} & 236 {\ttfamily ms} & 437 {\ttfamily ms} & - & 236 {\ttfamily ms} & 422 {\ttfamily ms} & - \\
{\ttfamily LOZO-m} & 40 {\ttfamily ms} & 78 {\ttfamily ms} & 104 {\ttfamily ms} & 181 {\ttfamily ms} & 312 {\ttfamily ms} & 677 {\ttfamily ms} & 180 {\ttfamily ms} & 316 {\ttfamily ms} & 759 {\ttfamily ms}\\
{\ttfamily TeZO-m} & 37 {\ttfamily ms} & 70 {\ttfamily ms} & 100 {\ttfamily ms} & 172 {\ttfamily ms} & 303 {\ttfamily ms} & 653 {\ttfamily ms} & 176 {\ttfamily ms} & 308 {\ttfamily ms} & 738 {\ttfamily ms}\\
\midrule
{\ttfamily MeZO-Adam} & 39 {\ttfamily ms} & 104 {\ttfamily ms} & 173 {\ttfamily ms} & 348 {\ttfamily ms} & 642 {\ttfamily ms} & - & 342 {\ttfamily ms} & 624 {\ttfamily ms} & - \\
{\ttfamily TeZO-Adam} & 54 {\ttfamily ms} & 92 {\ttfamily ms} & 134 {\ttfamily ms} & 224 {\ttfamily ms} & 394 {\ttfamily ms} & 841 {\ttfamily ms} & 227 {\ttfamily ms} & 397 {\ttfamily ms} & 937 {\ttfamily ms}\\
\bottomrule
\end{tabular}
\vskip -0.1in
\end{table}

From the perspective of memory, low-rank methods have consistently been effective in reducing memory usage. Whether on the OPT or LLaMA models, our {\ttfamily TeZO-Adam} method consistently incurs lower loss compared to the standard {\ttfamily MeZO} method, and uses approximately 30\% of the memory consumed by the {\ttfamily MeZO-Adam} method. 

From the perspective of wall-clock time, low-rank methods show a significant efficiency improvement on large models, while they perform poorly or even slower on smaller models. On the 125M model, low-rank methods is slower and on the 1.3B models, low-rank methods performs the same as {\ttfamily MeZO}. Since the model parameters are relatively small, the additional overhead of low-rank computation offsets the training cost. However, when the model size exceeds 3B, the efficiency improvement of low-rank methods becomes significant. Tests on both OPT and LLaMA models show that {\ttfamily TeZO-Adam} can achieve the same speed as {\ttfamily MeZO}, while being more than 1.5$\times$ faster than {\ttfamily MeZO-Adam}. 

These results are consistent with the \textit{Figure \ref{fg:memory and time}} in the main text. From the perspective of computational efficiency, we recommend: \textbf{it is better to adopt low-rank ZO methods on models larger than 3B to achieve valid improvements}.

\newpage
\subsection{Low-rank Parameters v.s. Low-rank ZO Methods}
Gradient low-rank approximation and model low-rank factorization are two key techniques for efficient training, as we mentioned earlier. Techniques like LoRA \cite{hu2021lora} and GaLore \cite{zhao2024galore}, they reduce the number of trainable model parameters and optimizer states through low-rank mapping and subspace mapping, respectively, thereby accelerating the training process. We want to emphasize that these two methods are orthogonal because they target different parameters, addressing the efficient training of different parts of the models during training. Recent works \cite{yu2024subzero,chen2024enhancing} apply low-rank ZO methods to train LoRA models, achieving some success. Here, we would like to emphasize that, according to the experimental records in Appendix \ref{ap:memory and time}, when the size of trainable parameters is too small, low-rank ZO methods provide almost no benefits. For instance, the LoRA model for the 13B model has approximately 300M parameters, and applying low-rank ZO at this parameter scale is clearly unnecessary. Therefore, in this part, we consider these two techniques as independent methods for comparisons, as shown in \textit{Table \ref{ap:tb:comparison lora and low-rank zo}}. The other setups are the same as above.

\begin{table}[H]
\centering
\vskip -0.1in
\caption{GPU memory usage (max memory reserved) for full fine-tuning, fine-tuning LoRA, fine-tuning prefix, and ZO methods.}
\vskip 0.05in
\label{ap:tb:comparison lora and low-rank zo}
\begin{tabular}{c|c|cc|cc}
\toprule
\midrule
\multicolumn{1}{c}{} & \multicolumn{1}{c}{} & \multicolumn{2}{c}{OPT-6.7B} & \multicolumn{2}{c}{OPT-13B} \\
\cmidrule(lr){3-4} \cmidrule(lr){5-6}
\multicolumn{1}{c}{} & \multicolumn{1}{c}{} & Memory & Ratio & Memory & Ratio \\
\midrule
& {\ttfamily ft} & 105.24 {\ttfamily G} & 8.27$\times$ & 238.26 {\ttfamily G} & 9.77$\times$ \\
{\ttfamily FO} & {\ttfamily ft-LoRA} & 37.96 {\ttfamily G} & 2.98$\times$ & 73.19 {\ttfamily G} & 3.00$\times$ \\
& {\ttfamily ft-prefix} & 38.23 {\ttfamily G} & 3.00$\times$ & 73.13 {\ttfamily G} & 3.00$\times$ \\
\midrule
\multirow{5.5}{*}{{\ttfamily ZO}} & {\ttfamily MeZO} & 14.40 {\ttfamily G} & 1.13$\times$ & 26.43 {\ttfamily G} & 1.08$\times$ \\
& {\ttfamily MeZO-LoRA} & 13.04 {\ttfamily G} & 1.02$\times$ & 24.82 {\ttfamily G} & 1.02$\times$ \\
& {\ttfamily MeZO-prefix} & 13.06 {\ttfamily G} & 1.03$\times$ & 24.81 {\ttfamily G} & 1.02$\times$ \\
\cmidrule(lr){2-6}
& {\ttfamily MeZO-Adam} & 39.98 {\ttfamily G} & 3.14$\times$ & 75.27 {\ttfamily G} & 3.09$\times$ \\
& {\ttfamily TeZO-Adam} & 14.07 {\ttfamily G} & 1.10$\times$ & 26.01 {\ttfamily G} & 1.06$\times$ \\
\midrule
& {\ttfamily Zero-Shot} & 12.73 {\ttfamily G} & 1$\times$ & 24.39 {\ttfamily G} & 1$\times$ \\
\midrule
\bottomrule
\end{tabular}
\vskip -0.1in
\end{table}
Compared to FO methods, the advantages of ZO methods remain significant. Even with methods of low-rank parameters, the memory usage is still nearly three times higher than ZO methods. Additionally, we want to emphasize that while ``ZO + LoRA" can further reduce training costs, the gains of memory-efficiency are negligible. Moreover, based on the experiments in existing studies, the performance of these approaches will significantly degrade on large models. ``ZO + fine-tuning full parameters" has already achieved to the comparable memory usage of zero-shot (inference only), and combining ZO with LoRA can only save very limited memory. Therefore, we do not advocate directly combining ZO methods with PEFT approaches. From the perspective of memory usage, the benefits of such a combination are indeed limited.

\newpage
\section{Proofs of Main Theorems.}
\label{ap:proofs}
\subsection{Proofs of Theorem~\ref{thm:mean and variance}}
We consider the mean at first.
According to Proposition A.1 proposed by \citet{chen2024enhancing}, we have:
\begin{align*}
    \lim_{\rho\rightarrow 0}\frac{f(W + \rho Z,\xi) - f(W,\xi) - \left\langle\nabla f(W,\xi),\rho Z\right\rangle}{\rho} = 0.
\end{align*}
Without loss of generality, we consider the case where the parameters are 2D matrix. On each step, we sample $\tau\sim\mathcal{N}(0,I_r)$ and compute the perturbation $Z=\sum_{s=1}^r\tau_s\cdot (u_s\circ v_s)$. By directly expanding the ZO gradient in {\ttfamily TeZO}, we have:
\begin{align*}
    &\quad \lim_{\rho\rightarrow 0}\nabla^0 f(w,\xi) = \lim_{\rho\rightarrow 0}\frac{f(W + \rho Z,\xi) - f(W - \rho Z,\xi)}{2\rho}\cdot Z \\
    &= \lim_{\rho\rightarrow 0}\frac{f(W + \rho Z,\xi) - f(W,\xi)-\left\langle\nabla f(W,\xi),\rho Z\right\rangle}{2\rho}\cdot Z - \lim_{\rho\rightarrow 0}\frac{f(W - \rho Z,\xi) - f(W,\xi)-\left\langle\nabla f(W,\xi),-\rho Z\right\rangle}{2\rho}\cdot Z \\
    &\quad + \lim_{\rho\rightarrow 0}\frac{\left\langle\nabla f(W,\xi), \rho Z\right\rangle}{\rho}
    \cdot Z = \left\langle\nabla f(W,\xi), Z\right\rangle\cdot Z,
\end{align*}
where the inner product performs as the calculation in vectors. With 
$Z$ substituted, the following holds:
\begin{align*}
    &\quad \lim_{\rho\rightarrow 0}\nabla^0 f(W,\xi) = \left\langle\nabla f(W,\xi), \sum_{s=1}^r\tau_s\cdot (u_s\circ v_s)\right\rangle\cdot \sum_{s=1}^r\tau_s\cdot (u_s\circ v_s).
\end{align*}
Specifically, we consider the element $\left[\lim_{\rho\rightarrow 0}\nabla^0 f(w,\xi)\right]_{i^\star,j^\star}$. To simplify the expression, we have slightly abused the notation $u_{s,i}$ and $v_{s,j}$, which means the $i$-th element in vector $u_s$ and $j$-th element in vector $v_s$. By taking the expectation,
\begin{align*}
    &\quad \mathbb{E}\left[\lim_{\rho\rightarrow 0}\nabla^0 f(W,\xi)\right]_{i^\star,j^\star} = \mathbb{E}\left\langle\nabla f(W,\xi), \sum_{s=1}^r\tau_s\cdot (u_s\circ v_s)\right\rangle\cdot \sum_{s=1}^r\tau_s u_{s,i^\star} v_{s,j^\star} \\
    &= \mathbb{E}\sum_{i,j} \left(\nabla f(W,\xi)_{i,j}\sum_{s=1}^r\tau_s u_{s,i} v_{s,j}\right)\cdot \sum_{s=1}^r\tau_s u_{s,i^\star} v_{s,j^\star} \\
    &= \underbrace{\mathbb{E}\sum_{i\neq i^\star,j\neq j^\star}\nabla f(W,\xi)_{i,j}\sum_{s,s'}^{r}\tau_s\tau_{s'}u_{s,i}u_{s',i^\star}v_{s,j}v_{s',j^\star}}_{\mathbb{E}_{u,v}\left[u_{s,i}u_{s',i^\star}v_{s,j}v_{s',j^\star}\right]=0} + \underbrace{\mathbb{E}\sum_{i=i^\star,j\neq j^\star}\nabla f(W,\xi)_{i,j}\sum_{s,s'}^{r}\tau_s\tau_{s'}u_{s,i}u_{s',i}v_{s,j}v_{s',j^\star}}_{\mathbb{E}_{v}\left[v_{s,j}v_{s',j^\star}\right]=0} \\
    &\quad + \underbrace{\mathbb{E}\sum_{i\neq i^\star,j=j^\star}\nabla f(W,\xi)_{i,j}\sum_{s,s'}^{r}\tau_s\tau_{s'}u_{s,i}u_{s',i^\star}v_{s,j}v_{s',j}}_{\mathbb{E}_u\left[u_{s,i}u_{s',i^\star}\right]=0} + \nabla f(W,\xi)_{i^\star,j^\star}\mathbb{E}\sum_{s,s'}^{r}\tau_s\tau_{s'}u_{s,i^\star}u_{s',i^\star}v_{s,j^\star}v_{s',j^\star} \\
    &= \nabla f(W,\xi)_{i^\star,j^\star}\underbrace{\mathbb{E}\sum_{s\neq s'}^{r}\tau_s\tau_{s'}u_{s,i^\star}u_{s',i^\star}v_{s,j^\star}v_{s',j^\star}}_{\mathbb{E}_{\tau}\left[\tau_s\tau_{s'}\right]=0} + \nabla f(W,\xi)_{i^\star,j^\star}\underbrace{\mathbb{E}\sum_{s=1}^{r}\tau_{s}^2u_{s,i^\star}^2v_{s,j^\star}^2}_{=r} = r \nabla f(W,\xi)_{i^\star,j^\star}.
\end{align*}
Clearly, when the SPSA form is directly applied, the expectation of the {\ttfamily TeZO} gradient becomes $r$ times the FO gradient. Therefore, by dividing $r$, {\ttfamily TeZO} is an unbiased estimation of the FO gradient.

Then we consider the variance. We have the following term:
\begin{align*}
    &\quad \mathbb{E}\Vert \frac{1}{r}\lim_{\rho\rightarrow 0}\nabla^0 f(w,\xi) - \nabla f(W,\xi)\Vert^2 = \frac{1}{r^2}\mathbb{E}\Vert\lim_{\rho\rightarrow 0}\nabla^0 f(w,\xi)\Vert^2 - \mathbb{E}\Vert \nabla f(W,\xi)\Vert^2\\
    &= \frac{1}{r^2}\mathbb{E}\Vert \left\langle\nabla f(W,\xi), \sum_{s=1}^r\tau_s\cdot (u_s\circ v_s)\right\rangle\cdot \sum_{s=1}^r\tau_s\cdot (u_s\circ v_s) \Vert^2 - \mathbb{E}\Vert \nabla f(W,\xi)\Vert^2 \\
    &= \frac{1}{r^2}\mathbb{E}\underbrace{\left\langle\nabla f(W,\xi), \sum_{s=1}^r\tau_s\cdot (u_s\circ v_s)\right\rangle^2}_{A} \cdot \underbrace{\left\langle \sum_{s=1}^r\tau_s\cdot (u_s\circ v_s),\sum_{s=1}^r\tau_s\cdot (u_s\circ v_s) \right\rangle}_{B} - \mathbb{E}\Vert \nabla f(W,\xi)\Vert^2.
\end{align*}
Let $g_{ij}=\nabla f(W,\xi)_{i,j}$ for convenience, we have:
\begin{align*}
    A
    &=\left\langle\nabla f(W,\xi), \sum_{s=1}^r\tau_s\cdot (u_s\circ v_s)\right\rangle^2 = \sum_{i,i'}\sum_{j,j'}\sum_{s,s'} g_{i,j}g_{i',j'}\tau_s\tau_{s'} u_{s,i}u_{s',i'}v_{s,j}v_{s',j'} \\
    &= \underbrace{\sum_{i\neq i'}\sum_{j\neq j'}\sum_{s\neq s'} g_{i,j}g_{i',j'}\tau_s\tau_{s'} u_{s,i}u_{s',i'}v_{s,j}v_{s',j'}}_{A_1} + \underbrace{\sum_{i\neq i'}\sum_{j\neq j'}\sum_{s} g_{i,j}g_{i',j'}\tau_s^2 u_{s,i}u_{s,i'}v_{s,j}v_{s,j'}}_{A_2} \\
    &\quad + \underbrace{\sum_{i\neq i'}\sum_{j}\sum_{s\neq s'} g_{i,j}g_{i',j}\tau_s\tau_{s'} u_{s,i}u_{s',i'}v_{s,j}v_{s',j}}_{A_3} + \underbrace{\sum_{i\neq i'}\sum_{j}\sum_{s} g_{i,j}g_{i',j}\tau_s^2 u_{s,i}u_{s,i'}v_{s,j}v_{s,j}}_{A_4} \\
    &\quad + \underbrace{\sum_{i}\sum_{j\neq j'}\sum_{s\neq s'} g_{i,j}g_{i,j'}\tau_s\tau_{s'} u_{s,i}u_{s',i}v_{s,j}v_{s',j'}}_{A_5} + \underbrace{\sum_{i}\sum_{j\neq j'}\sum_{s} g_{i,j}g_{i,j'}\tau_s^2 u_{s,i}u_{s,i}v_{s,j}v_{s,j'}}_{A_6} \\
    &\quad + \underbrace{\sum_{i}\sum_{j}\sum_{s\neq s'} g_{i,j}^2\tau_s\tau_{s'} u_{s,i}u_{s',i}v_{s,j}v_{s',j}}_{A_7} + \underbrace{\sum_{i}\sum_{j}\sum_{s} g_{i,j}^2\tau_s^2 u_{s,i}^2 v_{s,j}^2}_{A_8}.
\end{align*}
\begin{align*}
    B
    &=\left\langle \sum_{s=1}^r\tau_s\cdot (u_s\circ v_s),\sum_{s=1}^r\tau_s\cdot (u_s\circ v_s) \right\rangle = \sum_{i}\sum_{j}\sum_{s,s'}\tau_s\tau_{s'}u_{s,i}u_{s',i}v_{s,j}v_{s',j} \\
    &= \underbrace{\sum_{i}\sum_{j}\sum_{s\neq s'}\tau_s\tau_{s'}u_{s,i}u_{s',i}v_{s,j}v_{s',j}}_{B_1} + \underbrace{\sum_{i}\sum_{j}\sum_{s}\tau_s^2 u_{s,i}^2 v_{s,j}^2}_{B_2}.
\end{align*}
Similar to the way of computing expectations for the mean above, When there are cross terms like $u_{s,i}$ or $v_{s,j}$ in the product of $A_i$ and $B_j$, then $\mathbb{E}_{u,v}\left[A_iB_j\right]=0$. Therefore, it is easy to check that $\mathbb{E}_{u,v}\left[A_1B\right] = \mathbb{E}_{u,v}\left[A_2B\right] = \mathbb{E}_{u,v}\left[A_3B\right] = \mathbb{E}_{u,v}\left[A_4B\right] = \mathbb{E}_{u,v}\left[A_5B\right] = \mathbb{E}_{u,v}\left[A_6B\right] = 0$ and we have $\mathbb{E}_{u,v}\left[AB\right]=\mathbb{E}_{u,v}\left[(A_7+A_8)(B_1+B_2)\right]$. Then we consider the cross terms on $\tau_s$. In $A_8B_1$ and $A_7B_2$, there exist the independent $\tau_s$ term, that is, $\mathbb{E}_{\tau}\left[A_8B_1\right] = \mathbb{E}_{\tau}\left[A_7B_2\right] = 0$, and the expectation of $AB$ is $\mathbb{E}_{\tau,u,v}\left[AB\right] = \mathbb{E}_{\tau,u,v}\left[A_7B_1 + A_8B_2\right]$. For the first term, we have:
\begin{align*}
    \mathbb{E}\left[A_7B_1\right] = \mathbb{E}\left[2\sum_{i}\sum_{j}\sum_{s\neq s'} g_{i,j}^2\tau_s^2 \tau_{s'}^2 u_{s,i}^2 u_{s',i}^2 v_{s,j}^2 v_{s',j}^2\right] = 2\sum_{i}\sum_{j}\sum_{s\neq s'}g_{i,j}^2 = 2r(r-1)\sum_{i}\sum_{j}g_{i,j}^2.
\end{align*}
For the second term, we have:
\begin{align*}
    \mathbb{E}\left[A_8B_2\right] 
    &= \mathbb{E}\left[\sum_{i,i'}\sum_{j,j'}\sum_{s,s'} g_{i,j}^2\tau_s^2 \tau_{s'}^2 u_{s,i}^2 u_{s',i'}^2 v_{s,j}^2 v_{s',j'}^2\right] \\
    & = \mathbb{E}\left[\sum_{i\neq i'}\sum_{j\neq j'}\sum_{s\neq s'} g_{i,j}^2\tau_s^2 \tau_{s'}^2 u_{s,i}^2 u_{s',i'}^2 v_{s,j}^2 v_{s',j'}^2\right] + \mathbb{E}\left[\sum_{i\neq i'}\sum_{j\neq j'}\sum_{s} g_{i,j}^2\tau_s^4 u_{s,i}^2 u_{s,i'}^2 v_{s,j}^2 v_{s,j'}^2\right] \\
    &\quad + \mathbb{E}\left[\sum_{i\neq i'}\sum_{j}\sum_{s\neq s'} g_{i,j}^2\tau_s^2 \tau_{s'}^2 u_{s,i}^2 u_{s',i'}^2 v_{s,j}^2 v_{s',j}^2\right] + \mathbb{E}\left[\sum_{i\neq i'}\sum_{j}\sum_{s} g_{i,j}^2\tau_s^4 u_{s,i}^2 u_{s,i'}^2 v_{s,j}^4\right] \\
    &\quad + \mathbb{E}\left[\sum_{i}\sum_{j\neq j'}\sum_{s\neq s'} g_{i,j}^2\tau_s^2 \tau_{s'}^2 u_{s,i}^2 u_{s',i}^2 v_{s,j}^2 v_{s',j'}^2\right] + \mathbb{E}\left[\sum_{i}\sum_{j\neq j'}\sum_{s} g_{i,j}^2\tau_s^4 u_{s,i}^4 v_{s,j}^2 v_{s,j'}^2\right] \\
    &\quad + \mathbb{E}\left[\sum_{i}\sum_{j}\sum_{s\neq s'} g_{i,j}^2\tau_s^2 \tau_{s'}^2 u_{s,i}^2 u_{s',i}^2 v_{s,j}^2 v_{s',j}^2\right] + \mathbb{E}\left[\sum_{i}\sum_{j}\sum_{s} g_{i,j}^2\tau_s^4 u_{s,i}^4 v_{s,j}^4\right] \\
    &= \sum_{i\neq i'}\sum_{j\neq j'}\sum_{s\neq s'} g_{i,j}^2 + \sum_{i\neq i'}\sum_{j\neq j'}\sum_{s} 3 g_{i,j}^2 + \sum_{i\neq i'}\sum_{j}\sum_{s\neq s'} g_{i,j}^2 + \sum_{i\neq i'}\sum_{j}\sum_{s} 9 g_{i,j}^2 \\
    &\quad + \sum_{i}\sum_{j\neq j'}\sum_{s\neq s'} g_{i,j}^2 + \sum_{i}\sum_{j\neq j'}\sum_{s} 9 g_{i,j}^2 + \sum_{i}\sum_{j}\sum_{s\neq s'} g_{i,j}^2 + \sum_{i}\sum_{j}\sum_{s} 27 g_{i,j}^2 \\
    &= \left(mnr^2 + 2mnr + 6mr + 6nr + 12r\right)\sum_{i}\sum_{j} g_{i,j}^2.
\end{align*}
Thus, we can consolidate all the results as follows:
\begin{align*}
    &\quad \ \mathbb{E}\Vert \frac{1}{r}\lim_{\rho\rightarrow 0}\nabla^0 f(w,\xi) - \nabla f(W,\xi)\Vert^2 = \frac{1}{r^2}\mathbb{E}\left[A \cdot B\right] - \Vert \nabla f(W,\xi)\Vert^2 = \frac{1}{r^2}\mathbb{E}\left[A_7B_1 + A_8B_2\right] - \Vert \nabla f(W,\xi)\Vert^2 \\
    &= \left(1 + mn + \frac{2mn}{r} + \frac{6(m+n)}{r} + \frac{10}{r}\right)\Vert \nabla f(W,\xi)\Vert^2.
\end{align*}
This completes the proofs.

\subsection{Proofs of Theorem \ref{thm:convergence}}
We first introduce some basic lemmas for the subsequent proofs. In fact, when considering the properties of the function at each layer, we treat the parameters and gradients as a 2D matrices. However, to consider its general property, we treat them as a flattened parameter vector concatenation across layers. Therefore, in our proof, we slightly abuse both uppercase and lowercase letters, e.g., $\nabla f(Z)$ and $\nabla f(z)$, to express the specific properties of the gradient.
\begin{lemma}
    Under Assumption \ref{as:smooth} and \ref{as:stochastic}, ZO gradient of {\ttfamily TeZO} is an unbiased estimator of the full FO gradient $\nabla f(W)$ with the variance:
    \begin{equation}
        \mathbb{E}\Vert \frac{1}{r}\nabla^0 f(W,\xi) - \nabla f(W) \Vert^2 \leq \rho^2\lambda^2\delta_\rho + (\delta + 1)\sigma^2 + \delta\mathbb{E}\Vert\nabla f(W)\Vert^2,
    \end{equation}
    where $\delta=1 + mn + \frac{2mn}{r} + \frac{6(m+n)}{r} + \frac{10}{r}$ and $\delta_\rho=\frac{15r^2(m+3)^3(n+3)^3 + 36r^2(r-1)m^3n^3 + r^2(r-1)(r-2)m^3n^3}{4}$ are two constants.
\end{lemma}
\begin{proof}
    According to the studies of \citet{nesterov2017random,chen2024enhancing,yu2024subzero}, we first consider the smoothness property as follows:
    \begin{align*}
        f(W + \rho Z,\xi) - f(W,\xi) - \langle \nabla f(W,\xi), \rho Z \rangle \leq \frac{\lambda}{2}\Vert \rho Z \Vert^2 = \frac{\rho^2 \lambda}{2}\Vert Z \Vert^2.
    \end{align*}
Then we learn the distance between $\nabla^0 f(W)$ and the $\lim_{\rho\rightarrow 0}\nabla^0 f(W)$. Specifically, we consider the unbiased form as:
\begin{align*}
    & \quad \ \Vert \frac{1}{r}\nabla^0 f(W,\xi) - \frac{1}{r}\lim_{\rho\rightarrow 0}\nabla^0 f(W,\xi) \Vert^2 \\
    & = \frac{1}{r^2}\Vert \frac{f(W + \rho Z,\xi) - f(W - \rho Z,\xi)}{2\rho} \cdot Z - \langle\nabla f(W,\xi), Z\rangle\cdot Z \Vert^2 \\
    &= \frac{1}{r^2}\Vert\frac{f(W + \rho Z,\xi) - f(W,\xi) + f(W,\xi) - f(W - \rho Z,\xi) - 2\langle\nabla f(W,\xi), \rho z\rangle}{2\rho} \cdot Z \Vert^2 \\
    &= \frac{1}{r^2}\left\vert\frac{\left( f(W + \rho Z,\xi) - f(W,\xi) - \langle\nabla f(W,\xi), \rho Z\rangle\right) - \left( f(W - \rho Z,\xi) - f(W,\xi) - \langle\nabla f(W,\xi), -\rho Z\rangle \right)}{2\rho}\right\vert^2 \cdot \Vert Z \Vert^2 \\
    &\leq \frac{\rho^2 \lambda^2}{4r^2}\Vert Z \Vert^6.
\end{align*}
Substituting $Z=\sum_{s=1}^r \tau_s \cdot u_s \circ v_s$ and taking the expectation, we have:
\begin{align*}
    &\quad \ \mathbb{E}\Vert \frac{1}{r}\nabla^0 f(W,\xi) - \frac{1}{r}\lim_{\rho\rightarrow 0}\nabla^0 f(W,\xi) \Vert^2
    \leq \frac{\rho^2 \lambda^2}{4r^2}\mathbb{E}\Vert Z \Vert^6 = \frac{\rho^2 \lambda^2}{4r^2}\mathbb{E}\Vert \sum_{s=1}^r \tau_s \cdot u_s \circ v_s \Vert^6 \\
    &= \frac{\rho^2 \lambda^2}{4r^2}\mathbb{E}\left(\Vert \sum_{s=1}^r \tau_s \cdot u_s \circ v_s \Vert^2\right)^3 \leq \frac{\rho^2 \lambda^2}{4r^2}\mathbb{E}\left(r\sum_{s=1}^r\tau_s^2\Vert u_s \circ v_s \Vert^2\right)^3 = \frac{r \rho^2 \lambda^2}{4}\mathbb{E}\left(\sum_{s=1}^r\tau_s^2\Vert u_s \Vert^2 \Vert v_s \Vert^2\right)^3.
\end{align*}
Similarly, we can expand the term as:
\begin{align*}
    &\quad \ \mathbb{E}\left(\sum_{s=1}^r\tau_s^2\Vert u_s \Vert^2 \Vert v_s \Vert^2\right)^3 = \mathbb{E}\sum_{s}\sum_{s'}\sum_{s''} \tau_s^2 \tau_{s'}^2 \tau_{s''}^2 \Vert u_s \Vert^2 \Vert u_{s'} \Vert^2 \Vert u_{s''} \Vert^2 \Vert v_s \Vert^2 \Vert v_{s'} \Vert^2 \Vert v_{s''} \Vert^2 \\
    &= \mathbb{E}\sum_{s}\sum_{s'= s}\sum_{s''=s'}\tau_s^6\Vert u_s \Vert^6 \Vert v_s \Vert^6 + \mathbb{E}\sum_{s}\sum_{s'= s}\sum_{s''\neq s'}\tau_s^4 \tau_{s''}^2\Vert u_s \Vert^4 \Vert u_{s''} \Vert^2 \Vert v_s \Vert^4 \Vert v_{s''} \Vert^2 \\
    &\quad + \mathbb{E}\sum_{s}\sum_{s'\neq s}\sum_{s''=s}\tau_s^4 \tau_{s'}^2 \Vert u_s \Vert^4 \Vert u_{s'} \Vert^2 \Vert v_s \Vert^4\Vert v_{s'} \Vert^2 + \mathbb{E}\sum_{s}\sum_{s'\neq s}\sum_{s''=s'}\tau_s^2\tau_{s'}^4 \Vert u_s \Vert^2\Vert u_{s'} \Vert^4 \Vert v_s \Vert^2 \Vert v_{s'} \Vert^4 \\
    &\quad + \mathbb{E}\sum_{s}\sum_{s'\neq s}\sum_{s''\neq s,s'} \tau_s^2 \tau_{s'}^2 \tau_{s''}^2 \Vert u_s \Vert^2 \Vert u_{s'} \Vert^2 \Vert u_{s''} \Vert^2 \Vert v_s \Vert^2 \Vert v_{s'} \Vert^2 \Vert v_{s''} \Vert^2.
\end{align*}
Then we will discuss each term one by one. Actually, since $\tau,u,v$ are independent from each other, the expectation can be separated term by term. Since $u_s\sim\mathcal{N}(0,I_m)$, $v_s\sim\mathcal{N}(0,I_n)$ and $\tau_s\sim\mathcal{N}(0,1)$, we have: $\mathbb{E}\Vert u_s\Vert^2 = m$, $\mathbb{E}\Vert u_s\Vert^4=m(2m-1) \leq 2m^2$, $\mathbb{E}\Vert u_s\Vert^6 = m(15+3(m-1)+(m-1)(m-2)) \leq (m+3)^3$, $\mathbb{E}\Vert v_s\Vert^2 = n$, $\mathbb{E}\Vert v_s\Vert^4=n(2n-1) \leq 2n^2$, $\mathbb{E}\Vert v_s\Vert^6 = n(15+3(n-1)+(n-1)(n-2)) \leq (n+3)^3$, $\mathbb{E}\left[\tau_s^2\right] = 1$, $\mathbb{E}\left[\tau_s^4\right] = 3$ and $\mathbb{E}\left[\tau_s^6\right] = 15$. Therefore, we can provide the upper bound:
\begin{align*}
    \mathbb{E}\left(\sum_{s=1}^r\tau_s^2\Vert u_s \Vert^2 \Vert v_s \Vert^2\right)^3
    & \leq 15r(m+3)^3(n+3)^3 + 36r^2m^3n^3 + r^3m^3n^3.
\end{align*}
Let $\delta_\rho=\frac{15r^2(m+3)^3(n+3)^3 + 36r^3m^3n^3 + r^4m^3n^3}{4}$, then we have:
\begin{align*}
    \mathbb{E}\Vert \frac{1}{r}\nabla^0 f(W,\xi) - \frac{1}{r}\lim_{\rho\rightarrow 0}\nabla^0 f(W,\xi) \Vert^2
    \leq \frac{r \rho^2 \lambda^2}{4}\mathbb{E}\left(\sum_{s=1}^r\tau_s^2\Vert u_s \Vert^2 \Vert v_s \Vert^2\right)^3 \leq \rho^2 \lambda^2 \delta_\rho.
\end{align*}
Combining it with the variance in Theorem~\ref{thm:mean and variance}, we can finish the proofs.
\end{proof}

Then we can easily solve the convergence for {\ttfamily TeZO}. Similarly, without loss of generality, we still consider the 2D parameters. Let $\eta \leq \frac{1}{\lambda(\delta+1)}$ By expanding the smoothness inequality, we have:
\begin{align*}
    \mathbb{E}_t \left[ f(W_{t+1}) \right]
    &\leq f(W_t) + \mathbb{E}_t \langle \nabla f(W_t), W_{t+1} - W_t \rangle + \frac{\lambda}{2}\mathbb{E}_t \Vert W_{t+1} - W_t \Vert^2 \\
    &= f(W_t) + \eta \mathbb{E}_t\langle \nabla f(W_t), -G_t \rangle + \frac{\lambda\eta^2}{2}\mathbb{E}_t \Vert G_t \Vert^2 \\
    &= f(W_t) - \eta\mathbb{E}_t\Vert \nabla f(W_t) \Vert^2 + \frac{\lambda\eta^2}{2}\mathbb{E}_t \Vert G_t \Vert^2 \\ 
    &\leq f(W_t) - \eta\mathbb{E}_t\Vert \nabla f(W_t) \Vert^2 + \frac{\lambda\eta^2}{2}\mathbb{E}_t \Vert \frac{1}{r}\nabla^0 f(W_t,\xi) - \nabla f(W_t) \Vert^2 + \frac{\lambda\eta^2}{2}\mathbb{E}_t \Vert \nabla f(W_t) \Vert^2 \\ 
    &\leq f(W_t) - \eta\left(1 - \frac{\lambda(1+\delta)\eta}{2} \right)\mathbb{E}_t\Vert \nabla f(W_t) \Vert^2 + \eta^2\rho^2\frac{\lambda^3\delta_\rho}{2} + \eta^2\frac{\lambda(\delta + 1)\sigma^2}{2} \\
    &\leq f(W_t) - \frac{\eta}{2}\mathbb{E}_t\Vert \nabla f(W_t) \Vert^2 + \eta^2\rho^2\frac{\lambda^3\delta_\rho}{2} + \eta^2\frac{\lambda(\delta + 1)\sigma^2}{2}.
\end{align*}
Therefore, let $D_0 = f(W_0)-f(W_\star)$ be the initialized bias where $W_\star$ is the optimal solution, by accumulating it from $t=0$ to $T-1$ and taking the full expectation, we have:
\begin{align*}
    \frac{1}{T}\sum_{t=0}^{T-1}\mathbb{E}\Vert \nabla f(W_t) \Vert^2 \leq \frac{2D_0}{\eta T} + \eta\lambda\left(\rho^2\lambda^2\delta_\rho + (\delta + 1)\sigma^2\right).
\end{align*}
By simply selecting the learning rate $\eta=\mathcal{O}\left(\sqrt{\frac{D_0}{\lambda T\left(\rho^2\lambda^2\delta_\rho + \delta\sigma^2\right)}}\right)\leq \frac{1}{\lambda(\delta+1)}$, we have:
\begin{align*}
    \frac{1}{T}\sum_{t=0}^{T-1}\mathbb{E}\Vert \nabla f(W_t) \Vert^2 = \mathcal{O}\left(\sqrt{\frac{\lambda D_0\left(\rho^2\lambda^2\delta_\rho + \delta\sigma^2\right)}{T}}\right).
\end{align*}

\end{document}